\documentclass[10pt,twocolumn,letterpaper]{article}
\usepackage[pagenumbers]{cvpr}

\definecolor{cvprblue}{rgb}{0.21,0.49,0.74}

\usepackage[pagebackref,breaklinks,colorlinks,allcolors=cvprblue]{hyperref}
\usepackage{times}
\usepackage{epsfig}
\usepackage{graphicx}
\usepackage{amsmath,amssymb,amsthm}
\usepackage{booktabs}
\usepackage{adjustbox} 
\usepackage{array} 
\usepackage{multirow}
\usepackage{mathtools}
\usepackage{algorithm}
\usepackage{algorithmic}
\usepackage{bm}
\usepackage{enumitem}
\usepackage{color}
\usepackage{wrapfig}
\usepackage[accsupp]{axessibility}
\usepackage{microtype}
\usepackage{dsfont}
\usepackage{pifont}
\usepackage{hyperref}
\hypersetup{colorlinks,linkcolor=blue,citecolor=blue,urlcolor=blue}
\usepackage[table]{xcolor}
\usepackage{colortbl}

\newcommand{\PP}{\mathbb{P}}

\newtheorem{lemma}{Lemma}

\def\paperID{4598}
\def\confName{CVPR}
\def\confYear{2026}

\title{SG-OIF: A Stability-Guided Online Influence Framework for Reliable Vision Data}

\author{
Penghao Rao \qquad Runmin Jiang \qquad Min Xu\thanks{Corresponding author.} \\
Computational Biology Department, Carnegie Mellon University\\
{\tt\small penghaor@andrew.cmu.edu \qquad  runminj@andrew.cmu.edu \qquad mxu1@cs.cmu.edu}
}

\begin{document}
\maketitle
\begin{abstract}
Approximating training-point influence on test predictions is critical for deploying deep-learning vision models, essential for locating noisy data. Though the influence function was proposed for attributing how infinitesimal up-weighting or removal of individual training examples affects model outputs, its implementation is still challenging in deep-learning vision models: inverse-curvature computations are expensive, and training non-stationarity invalidates static approximations. Prior works use iterative solvers and low-rank surrogates to reduce cost, but offline computation lags behind training dynamics, and missing confidence calibration yields fragile rankings that misidentify critical examples. To address these challenges, we introduce a Stability-Guided Online Influence Framework (SG-OIF), the first framework that treats algorithmic stability as a real-time controller, which (i) maintains lightweight anchor IHVPs via stochastic Richardson and preconditioned Neumann; (ii) proposes modular curvature backends to modulate per-example influence scores using stability-guided residual thresholds, anomaly gating, and confidence. Experimental results show that SG-OIF achieves SOTA (State-Of-The-Art) on noise-label and out-of-distribution detection tasks across multiple datasets with various corruption. Notably, our approach achieves 91.1\% accuracy in the top 1\% prediction samples on the CIFAR-10 (20\% asym), and gets 99.8\% AUPR score on MNIST, effectively demonstrating that this framework is a practical controller for online influence estimation.
\end{abstract}    
\section{Introduction}
\label{sec:intro}

Modern vision systems are increasingly deployed in high-stakes, compliance-sensitive settings \cite{doshi2017towards, rudin2019stop}, where accountability must extend beyond aggregate accuracy to continuous, per-example oversight-surfacing mislabeled or low-quality samples \cite{northcutt2021confident}, detecting distribution shift \cite{rabanser2019failing, lipton2018detecting}, resisting data poisoning \cite{steinhardt2017certified} and enabling selective unlearning \cite{golatkar2020eternal, bourtoule2021machine}. Regulatory and operational requirements, such as the right to erasure and continual risk assessments \cite{regulation2018general} further demand streaming influence assessment rather than sporadic offline diagnostics. Workflow is shown in \cref{fig:sg_oif_workflow}

\begin{figure}[htbp]
\centering
\includegraphics[width=0.45\textwidth]{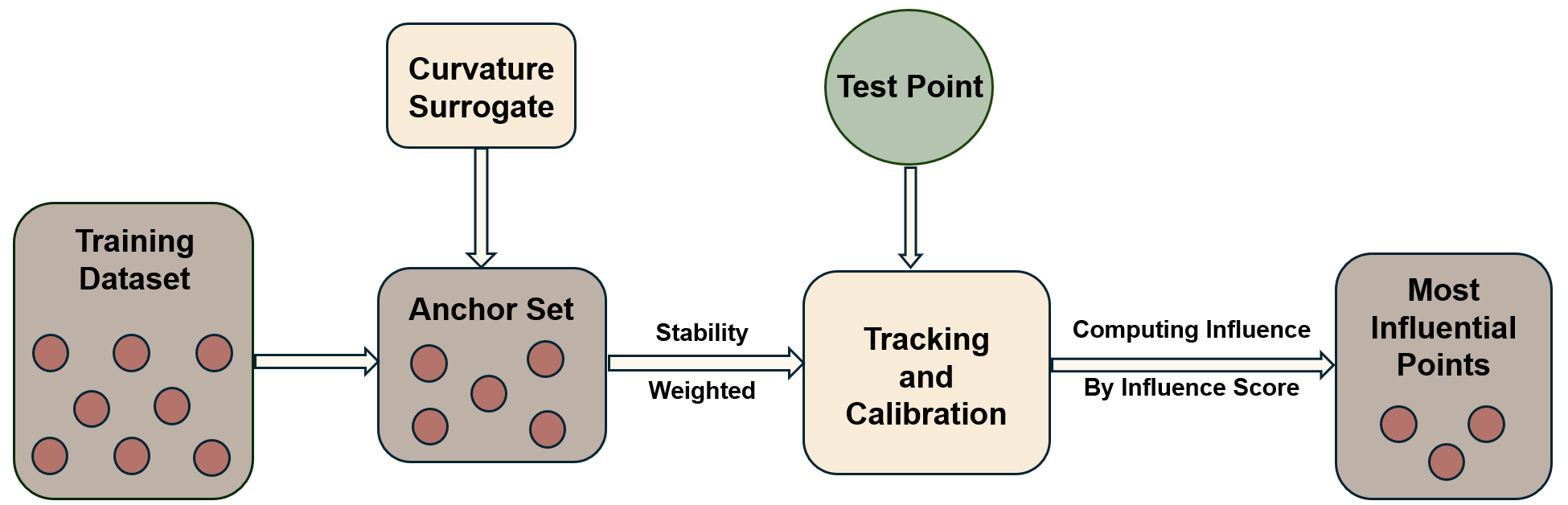}
\caption{Workflow of SG-OIF. The vature surrogate for the surrogate is updated online; then inverse Hessian vector proxies for anchors are tracked with lightweight iterations and calibrated by stability-based confidence; the influence of training data points is computed by aggregating stability-weighted per-anchor scores, with optional refinement when high-influence but low-confidence; finally, the most influential training data points are returned.}
\label{fig:sg_oif_workflow}
\end{figure}

Influence functions originate in robust statistics, quantifying the infinitesimal effect of perturbing one observation on an estimator \cite{jaeckel1972infinitesimal, hampel1974influence, cook1977detection, cook1980characterizations, pregibon1981logistic}. However, deploying influence estimation in modern deep vision models is challenging because it hinges on inverse-curvature computations that are expensive, ill-conditioned, and misaligned with the nonstationary dynamics of training. Koh and Liang \cite{koh2017understanding} adapted them to modern ML by approximating inverse Hessian–vector products with iterative second-order methods, enabling per-example attributions in high dimensions. However, this route inherits notable limitations in deep networks: Offline curvature quickly becomes stale along the training trajectory, and iterative Inverse Hessian Vector Product (IHVP) solvers are numerically fragile under non-convex saddle geometry. Further studies \cite{koh2019accuracy} evaluate offline settings on convex and near-convex surrogates, but they also face limitations: Without calibrated confidence, there is little guidance for anomaly gating in nonconvex deep networks.

Algorithmic stability provides a complementary lens for generalization aligned with governance. In regularized ERM \cite{koh2017understanding}, if replacing one point changes the loss by at most a small stability coefficient $\beta_n$, generalization scales with this coefficient; for strongly convex \cite{koh2019accuracy}, Lipschitz losses with Tikhonov regularization, this coefficient becomes tighter as regularization strengthens and the sample size grows. Extensions to iterative procedures analyze SGD perturbations across neighboring datasets, with $\beta_n$ controlled by cumulative step sizes and smoothness; in nonconvex settings, multiplicative smoothness inflates $\beta_n$. Robust training further stresses stability, as nonsmooth \cite{kusch2025second}, weakly convex objectives can make vanilla adversarial SGD non–uniformly stable, aligning with robust overfitting; smoothing and proximal designs \cite{t2020personalized} can help recover stability guarantees on the order typical of well-regularized regimes. Thus, the field evolves from convex ERM with explicit formulas, through stepwise analyses of stochastic optimization, to a stable restoring and shifting design principle.

Across governance tasks, the core need is to rank training examples by causal leverage with calibrated reliability. Existing tools face multiple challenges. Shapley-style valuation \cite{mehta2025exchangeability} is combinatorial and high-variance, straining budgets in large-scale settings-motivating lighter surrogates with variance control. Classical influence \cite{basu2020influence, alaa2020discriminative, du2014influence} is based on inverse curvature; IHVP pipelines are fragile in saddle-rich non-convex regimes, motivating curvature-aware updates with online refresh and stability gating. Trajectory heuristics \cite{koh2017understanding} entangle optimization transients with causality, mis-ranking late-stage hard cases-calling for drift-aware corrections and temporal consistency. Classical stability \cite{xiao2022stability} provides dataset-level worst-case bounds without per-example confidence-necessitating calibrated, instance-level confidence to gate large attributions. These gaps call for a streaming framework with modest overhead, curvature and drift aware computation, and per-example confidence.

To improve reliable, real-time causal ranking under nonconvex training, we introduce a stability-guided streaming framework that treats algorithmic stability as a confidence modulator for online inverse-curvature computation. At its core, we maintain a lightweight anchor bank and iteratively refine factored inverse Hessian vector directions using stochastic curvature sketching with preconditioned Neumann updates. These updates are steered by stability-inspired signals-gaps between Neumann and Richardson iterates, projected gradient energy, and the variance of stochastic Hessian probes, which adapt truncation depth, step sizes, and the growth of a low-rank subspace, while gating anomalies to suppress numerically unsupported large scores. Building on this, we calibrate per-example confidence via empirical-Bernstein intervals and local contraction checks, enabling residual-budgeted early stopping and principled score modulation. Finally, we expose task adaptors for label error triage, poisoning localization, shift sentinels, curriculum reweighting, and unlearning prioritization, and we ground the design with analysis under local Polyak-Lojasiewicz regularity and mild Hessian moment conditions, yielding a bias–variance decomposition, concentration bounds on misranking probability, and near-Pareto compute–accuracy trade-offs. The architecture of our framework is shown in \cref{fig:architecture_overview}.

The main contributions of this paper are:
\begin{itemize}
\item \textbf{Stability-Guided Online Influence Framework (SG-OIF):}
We unify influence functions with algorithmic stability by using stability as a reliability controller for online inverse-curvature approximation.

\item \textbf{Error Decomposition and Theoretical Guarantees:}
We develop an error decomposition and establish convergence under the Polyak-Lojasiewicz condition with concentration-style bounds linking stability proxies to misclassification probability.

\item \textbf{Empirical Results:}
Experiments across diverse benchmarks demonstrate that SG-OIF delivers higher detection quality and lower computational overhead in noise detection and achieves more stable and robust rankings on distribution shift monitoring.
\end{itemize}

\begin{figure*}[htbp]
\centering
\includegraphics[width=\textwidth]{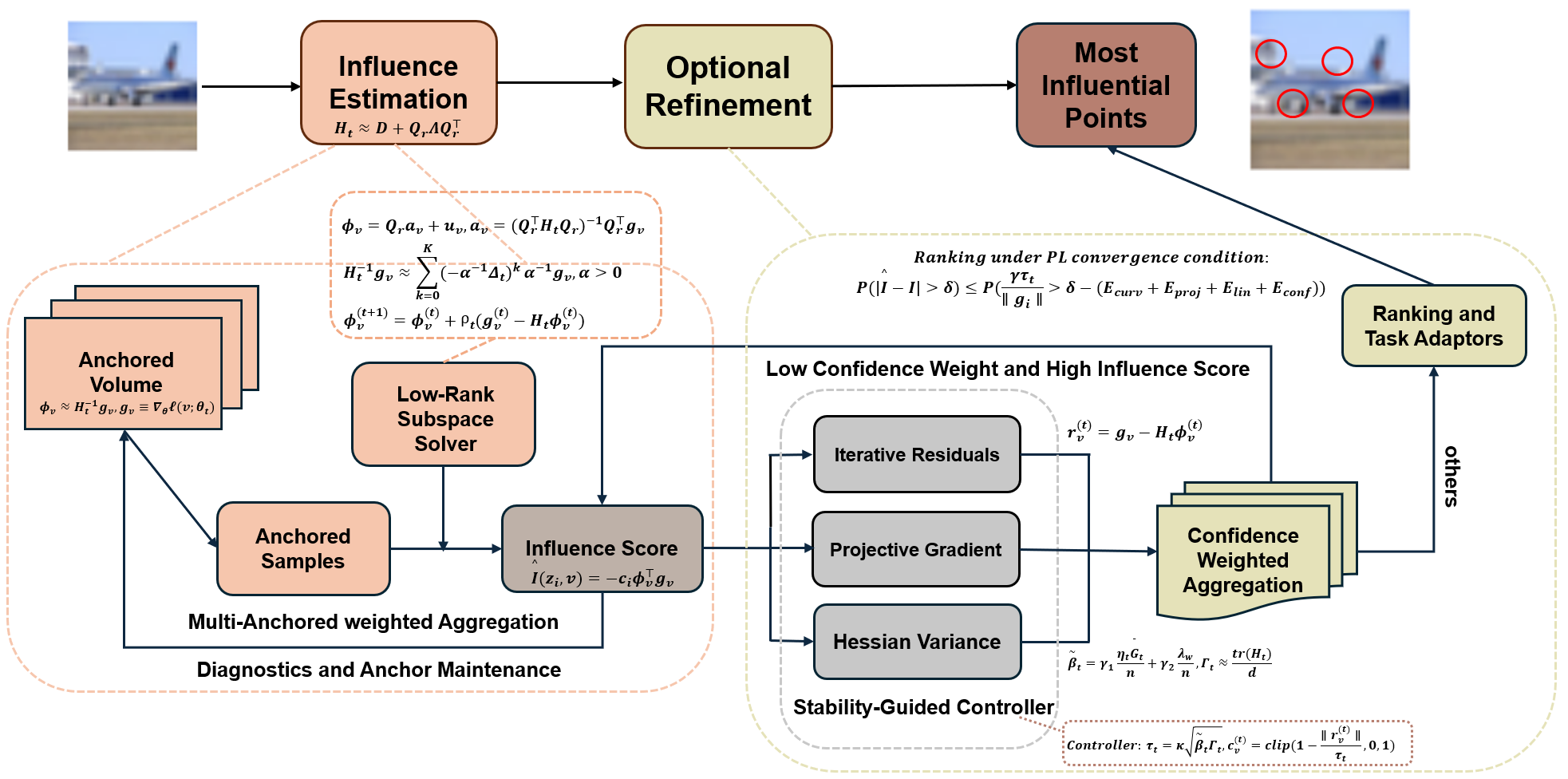}
\caption{Overview of SG-OIF Architecture. The pipeline begins with input, then computes influence scores via stability-guided scoring and returns the most influential points. In detail, for samples that exhibit high influence but low confidence weights, an optional refinement step is triggered to enhance the estimation. This step is so so-called Stability-Guided Controller. Influence is anchored: anchored samples flow back to the anchored volume. Multi-anchor weighted aggregation combines per-anchor results into a refined influence score. his step also performs diagnostics and anchor maintenance, which modifies the anchored volume. These calibrated scores feed ranking and task adaptors to select the most influential data and support downstream actions. The framework ensures reliable, robust influence estimation and output.}
\label{fig:architecture_overview}
\end{figure*}
\section{Related Work}
\label{sec:related}

\subsection{Algorithmic Stability}
Using stability tools is one of the key parts in reliable training. Rogers and Wagner \cite{rogers1978finite} first proposed this concept while Bousquet and Elisseeff \cite{bousquet2002stability} first put it into statistical learning problems by analyzing the algorithm-based generalization
bounds. Other works \cite{feldman2018generalization, xiao2022adversarial, t2020personalized} demonstrate the importance of uniform stability in small training-set perturbations bound loss changes and \cite{yin2019rademacher, awasthi2020adversarial} study the limit of this generalization bound, offering practical guidance for reliable training. However, they typically report aggregate generalization behavior rather than per-sample reliability signals. Later works \cite{raghunathan2019adversarial, madry2017towards, schmidt2018adversarially} have shown the necessity of tracking per-example and proposed reliable methods. But current works \cite{yin2019rademacher, awasthi2020adversarial, xiao2022adversarial, t2020personalized} still do not show that algorithm stability can be used to estimate the influence of training data. In our work, we first combine the algorithm stability with per-example trace, showing that the influence of each data point can be detected and controlled.

\subsection{Influence Functions}
The use of influence function can be traced back to the Cook et al. works \cite{cook1977detection, cook1980characterizations, cook1986assessment, cook1991r} on model diagnostics. Later, Koh and
Liang et al. \cite{koh2017understanding} propose that this can be used for approximating the effect of upweighting and removing individual training examples on parameters and predictions.  Previous studies show that second-order attributions are computationally feasible at scale: Pearlmutter et al. \cite{pearlmutter1994fast} demonstrate that Hessian-vector products can be computed exactly and efficiently, and Martens et al. \cite{martens2010deep} show that Hessian-free methods with conjugate gradients can leverage these products to avoid explicit Hessian inversion in large models. This brought the second-approximation influence function to large-scale deep learning and has been followed up by numerous publications \cite{koh2022stronger, schulam2019can, brunet2019understanding, koh2019accuracy, basu2020second, feldman2020neural}. Building on these computational primitives, researchers then introduce structured curvature to further reduce cost: To amortize costs across many queries and evolving models, Raskutti et al. \cite{raskutti2016statistical} show that randomized sketching can accelerate inverse operations and support batched or streaming settings, Ba et al. \cite{ba2017distributed} demonstrate that Kronecker-factored approximations preserve key curvature structure while dramatically lowering memory and compute, Zhao et al. \cite{zhao2016low} show that low-rank-plus-diagonal hybrids offer additional tractability for inverse applications. In parallel, to enhance numerical stability, Botev et al. \cite{botev2020gauss} demonstrate that Gauss--Newton approximations provide positive semi-definite surrogates that stabilize second-order updates and attributions, while Deng et al. \cite{deng2023uncertainty} show that Fisher-based approaches incorporate uncertainty-aware curvature that can improve robustness. Nevertheless, these modern models still suffer from a fragile inverse transformation of the Hessian matrix, and we address this problem by proposing a stability-guided gate.

\section{Methodology}
\label{sec:methodology}
\begin{figure}[htbp]
\centering
\includegraphics[width=\linewidth]{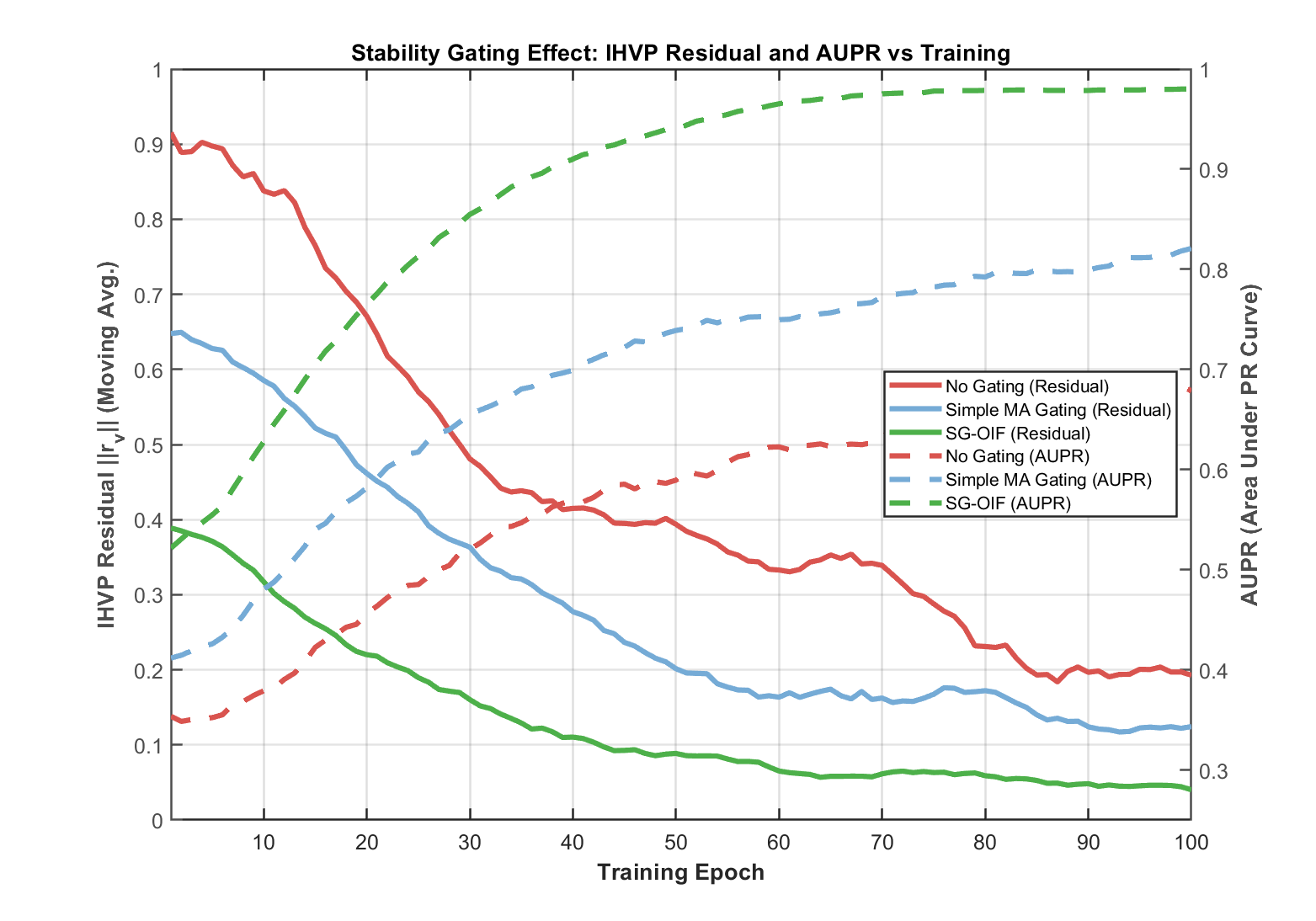}
\caption{\textbf{Residual Convergence and AUPR Improvement Under Stability Gating.} SG-OIF (green) achieves rapid residual (solid line) convergence and monotonic AUPR (dashed line) growth, outperforming no-gating (red) and simple MA gating (blue), confirming the necessity of stability monitoring for robust ranking.}
\label{fig:dual_axis}
\end{figure}

\subsection{SG-OIF Architecture}
\cref{fig:dual_axis} illustrates the core principle of our framework: stability-guided control enables reliable influence estimation throughout training by coupling residual monitoring with confidence modulation. \cref{alg:sgoif} shows the whole process of our method. More details are shown below.

Our framework maintains a compact bank of anchor inverse curvature directions and refines them with lightweight iterative updates while continuously auditing numerical support. Stability signals modulate how aggressively we update, when to refresh curvature, and how to modify the scores. Concretely, it has three complementary components:
\begin{itemize}
\item \textbf{Stability-Guided Influence Estimation.} We monitor solver residuals $r_v = g_v - H_t\phi_v$ and modulate each influence score by a confidence weight $c_v \in [0,1]$, suppressing attributions when numerical fidelity is low.
\item \textbf{Structured Curvature Backend.} We maintain modular curvature approximations that trade accuracy for speed and enable tractable inverse updates.
\item \textbf{Streaming IHVP Refinement.} We update anchor IHVPs $\{\phi_v\}$ via preconditioned fixed-point iterations, accelerated by a low-rank subspace $Q_r$ that captures dominant curvature directions and amortizes computation across anchors.
\end{itemize}
This makes influence estimation both responsive and trustworthy: stability-guided estimation suppresses unsupported extremes, curvature backend is refreshed as training evolves, and the IHVP refinement captures dominant directions.

\paragraph{Stability-Guided Influence Estimation.} We maintain IHVP per anchor $v$ and score each training example $z_i$ by
\begin{equation}
\hat{I}(z_i,v) = -c_v \phi_v^\top g_i, \quad c_v \in [0,1],
\label{eq:example_score}
\end{equation}
where $c_v$ is a confidence weight derived from the solver residual $r_v = g_v - H_t\phi_v$. To set the residual tolerance adaptively, we estimate a training stability proxy
\begin{equation}
\tilde{\beta}_t = \gamma_1 \frac{\eta_t \bar{G}_t}{n} + \gamma_2 \frac{\lambda_w}{n}
\label{eq:stability_proxy}
\end{equation}
where $\eta_t$ is the learning rate, $\bar{G}_t$ the gradient norm, and $\lambda_w$ the weight-decay coefficient. We couple $\tilde{\beta}_t$ with a curvature condition number proxy $\Gamma_t$ to define the tolerance
\begin{equation}
\tau_t = \kappa \tilde{\beta}_t \Gamma_t
\label{eq:tolearnce}
\end{equation}
and gate the confidence via
\begin{equation}
c_v^{(t)} = \mathrm{clip}\Big(1 - \frac{\|r_v^{(t)}\|}{\tau_t}, 0, 1\Big)
\label{eq:confidence_gate}
\end{equation}
This mechanism suppresses low-fidelity attributions in unstable and ill-conditioned phases. When a large magnitude score $|\phi_v^\top g_i|$ emerges with low confidence $c_v$, we trigger targeted refinement to elevate fidelity where it impacts ranking. Finally, per-example scores will be weighted:
\begin{equation}
\hat{I}(z_i) = \sum_{v\in\mathcal{V}} w_v \hat{I}(z_i,v), \quad w_v = \frac{c_v}{\sum_{u\in\mathcal{V}} c_u}
\label{eq:weighted_score}
\end{equation}

\paragraph{Structured Curvature Backend.}
We support four backends with increasing fidelity: Diagonal second moment approximations that scale to large models; Empirical Fisher matrices that improve conditioning via full batch statistics; Kronecker-factored blocks that capture layer-wise correlations; and Hybrid low-rank plus decompositions.

To maintain high-quality approximations, we ensure the anchor set $\mathcal{V}$ provides sufficient gradient coverage. Specifically, we construct the Gram matrix
\begin{equation}
G = \Phi^\top\Phi \in \mathbb{R}^{K \times K}
\label{eq:gram_matrix}
\end{equation}
where $\Phi = [\phi_{v_1}, \ldots, \phi_{v_K}]$ stacks the $\ell_2$-normalized anchor features. When approximating IHVP, the residual satisfies
\begin{equation}
\resizebox{0.9\columnwidth}{!}{$\displaystyle
\| H^{-1}g_z - \sum_{v \in \mathcal{V}} \alpha_v \phi_v \|^2 \leq \frac{1}{\lambda_{\min}(G)} \| g_z - \text{proj}_{\text{span}(\Phi)} g_z \|^2
$}
\label{eq:residual}
\end{equation}
and we refresh anchors when it drops below $0.1$.

\paragraph{Streaming IHVP Refinement.}
Refinement is separated into two parts: Fixed-Point Iteration with Preconditioning and Low-Rank Subspace Acceleration.

\textbf{Fixed-Point Iteration with Preconditioning.}
We refine each $\phi_v$ via stochastic Richardson iteration:
\begin{equation}
\phi_v^{(t+1)} = \phi_v^{(t)} + \rho_t \bigl( g_v^{(t)} - H_t \phi_v^{(t)} \bigr)
\label{eq:R_ite}
\end{equation}
where $\rho_t$ is a step size. When the curvature admits a diagonal-plus-perturbation structure $H_t = \alpha I + \Delta_t$ with $\alpha > 0$, we apply a preconditioned Neumann expansion
\begin{equation}
H_t^{-1} g_v \approx \sum_{k=0}^{K} \bigl(-\alpha^{-1}\Delta_t\bigr)^k \alpha^{-1} g_v
\label{eq:N_expan}
\end{equation}
We increase the truncation order $K$ adaptively: when residual norms are small, we expand further; when curvature is extremely jumping, we stop early to avoid amplifying errors.

\textbf{Low-Rank Subspace Acceleration.}
To amortize computation across anchors and training steps, we maintain a streaming rank-$r$ eigenspace $Q_r$ that captures dominant curvature directions. We decompose each IHVP as
\begin{equation}
\phi_v = Q_r a_v + u_v, \quad a_v = \bigl(Q_r^\top H_t Q_r\bigr)^{-1} Q_r^\top g_v
\label{eq:ihvp_decompostion}
\end{equation}
where $a_v$ is the in-subspace component and $u_v$ the residual. We update $u_v$ via a hybrid curvature model $H_t \approx D + Q_r \Lambda Q_r^\top$ that combines a diagonal approximation $D$ with the low-rank correction. This Woodbury-style factorization enables fast inverse applications and controlled acceleration.

\subsection{Theoretical Guarantee}
\paragraph{Hierarchical Error Decomposition.}
Let $e_v$ denote the IHVP error. The influence error follows the decomposition:
\begin{equation}
\resizebox{0.85\columnwidth}{!}{$\displaystyle
|\hat{I}(z_i,v)-I(z_i,v)| \le 
E_{\text{solver}} + E_{\text{curv}} + E_{\text{proj}} + E_{\text{lin}} + E_{\text{conf}}
$}
\label{eq:error_decompostion}
\end{equation}
with $E_{\text{solver}}=|e_v|,|g_i|$, curvature mismatch $E_{\text{curv}}$, subspace projection loss $E_{\text{proj}}$, higher-order drift $E_{\text{lin}}$, and confidence
\begin{equation}
E_{\text{conf}}=(1-c_v)|\phi_v^\top g_i|
\label{eq:E_cof}
\end{equation}
Given that $|e_v|\le \gamma |r_v|$, $|r_v|\le \tau_t$ and $\delta>0$, we obtain
\begin{equation}
\resizebox{0.85\columnwidth}{!}{$\displaystyle
\PP(|\hat{I}-I|>\delta) \le 
\PP\Big(\gamma \tau_t > \frac{\delta - (E_{\text{curv}}+E_{\text{proj}}+E_{\text{lin}}+E_{\text{conf}})}{\|g_i\|}\Big)
$}
\label{eq:condition}
\end{equation}
which explicitly links the stability scaling $\kappa$ (via $\tau_t$) to mis-ranking probability, thus enabling principled tuning.
\paragraph{Polyak-Lojasiewicz Convergence.}
Under a local Polyak - Lojasiewicz condition, the Richardson iterates exhibit contraction towards the IHVP up to a bound controlled by the dispersion of $H_t$. This ensures stabilized $\phi_v$ trajectories and well-tuned confidences $c_v$ throughout training.
\begin{algorithm}[htbp]
\caption{Stability-Gated Online Influence (SG-OIF)}
\label{alg:sgoif}
\textbf{Input:} Anchor set $\mathcal{V}$; inverse step schedule ${\rho_t}$; low-rank update period $T_r$; anchor refresh period $T_a$; stability scale $\kappa$\
\textbf{Output:} Per-example influence scores ${\hat{I}(z_i)}$; anchor states ${(\phi_v,c_v)}_{v\in\mathcal{V}}$
\begin{algorithmic}[1]
\STATE Initialize $\phi_v \gets 0$, $c_v \gets 0$ for all $v\in\mathcal{V}$
\FOR{$t=0,1,2,\dots$}
\STATE Sample minibatch $\mathcal{B}_t$; update model; build curvature surrogate $H_t$; \textbf{if} $t \bmod T_r = 0$ \textbf{ then } update low-rank basis $Q_r$
\STATE \textbf{For} $v\in\mathcal{V}$: compute $g_v$; update $\phi_v \gets \phi_v + \rho_t\big(g_v - H_t \phi_v\big)$; set $r_v \gets g_v - H_t \phi_v$; set $\tau_t \gets \kappa,\tilde{\beta}_t,\Gamma_t$; set $c_v \gets \mathrm{clip}!\bigl(1-|r_v|/\tau_t,,0,,1\bigr)$
\STATE \textbf{For} $z_i\in\mathcal{B}t$: compute $g_i$; set $\hat{I}(z_i) \gets \sum{v\in\mathcal{V}} c_v,\phi_v^\top g_i$
\STATE \textbf{If} $t \bmod T_a = 0$ \textbf{ then } replace lowest-confidence anchors; optionally refine low-$c_v$ / large $|\phi_v^\top g_i|$ via short conjugate-gradient
\ENDFOR
\end{algorithmic}
\end{algorithm}
\section{Experiments}
\label{sec:experiments}
\begin{table*}[htbp]
\centering
\scriptsize
\caption{P@1\% of SG-OIF versus recent SOTA methods for learning with noisy labels in CIFAR-10 \cite{krizhevsky2012imagenet}. When other methods decrease in performance from low sparsity (0.0) to high sparsity (0.6), SG-OIF consistently outperforms.}
\begin{adjustbox}{width=1.0\textwidth}
\begin{tabular}{l|cccc|cccc|cccc}
\toprule
\multirow{2}{*}{Method} & \multicolumn{4}{c|}{Sparsity (Noise 20\%)} & \multicolumn{4}{c|}{Sparsity (Noise 40\%)} & \multicolumn{4}{c}{Sparsity (Noise 70\%)} \\
\cmidrule(lr){2-5} \cmidrule(lr){6-9} \cmidrule(lr){10-13}
 & 0.0 & 0.2 & 0.4 & 0.6 & 0.0 & 0.2 & 0.4 & 0.6 & 0.0 & 0.2 & 0.4 & 0.6 \\
\midrule
INCV\cite{chen2019understanding} & 87.8 & 88.6 & 89.6 & 89.2 & 84.4 & 76.6 & 85.4 & 73.6 & 28.3 & 25.3 & 34.8 & 29.7 \\
Mixup\cite{zhang2017mixup} & 85.6 & 86.8 & 87.0 & 84.3 & 76.1 & 75.4 & 68.6 & 59.8 & 32.2 & 31.3 & 32.3 & 26.9 \\
SCE-loss\cite{wang2019symmetric} & 87.2 & 87.5 & 88.8 & 84.4 & 76.3 & 74.1 & 64.9 & 58.3 & 33.0 & 28.7 & 30.9 & 24.0 \\
MentorNet\cite{jiang2018mentornet} & 84.9 & 85.1 & 83.2 & 83.4 & 64.4 & 64.2 & 62.4 & 61.5 & 30.0 & 31.6 & 29.3 & 27.9 \\
Co-Teaching\cite{han2018co} & 81.2 & 81.3 & 81.4 & 80.6 & 62.9 & 61.6 & 60.9 & 58.1 & 30.5 & 30.2 & 27.7 & 26.0 \\
FASTIF\cite{guo2020fastif} & 80.0 & 80.0 & 79.7 & 79.1 & 58.6 & 61.2 & 59.1 & 57.5 & 28.4 & 28.5 & 27.9 & 27.3 \\
TrancIn\cite{pruthi2020estimating} & 78.1 & 78.9 & 80.8 & 79.3 & 60.5 & 60.4 & 61.2 & 58.6 & 29.0 & 29.4 & 29.1 & 26.8 \\
CL\cite{northcutt2021confident} & 78.4 & 79.2 & 79.0 & 78.2 & 60.2 & 60.8 & 59.6 & 57.3 & 27.0 & 29.7 & 28.2 & 26.8 \\
\midrule
\rowcolor{gray!20}\textbf{SG-OIF (ours)} & \textbf{91.1}$\uparrow$ & \textbf{90.9}$\uparrow$ & \textbf{91.1}$\uparrow$ & \textbf{91.3}$\uparrow$ & \textbf{87.1}$\uparrow$ & \textbf{86.9}$\uparrow$ & \textbf{86.7}$\uparrow$ & \textbf{87.2}$\uparrow$ & \textbf{41.1}$\uparrow$ & \textbf{41.8}$\uparrow$ & \textbf{39.1}$\uparrow$ & \textbf{36.4}$\uparrow$ \\ 
\bottomrule
\end{tabular}
\end{adjustbox}
\label{tab:noisy-main}
\end{table*}

\begin{table*}[htbp]
\centering
\caption{P@1\%, AUC-PR, and Overhead of SG-OIF across CIFAR-100, WebVision, and Clothing1M. Compared with the same SOTA models, SG-OIF achieves superior results with notably lower computational efficiency.}
\begin{adjustbox}{width=1.0\textwidth}
\begin{tabular}{l|ccc|ccc|ccc}
\toprule
\multirow{2}{*}{Method} & \multicolumn{3}{c|}{CIFAR-100 \cite{krizhevsky2012imagenet}} & \multicolumn{3}{c|}{WebVision \cite{li2017webvision}} & \multicolumn{3}{c}{Clothing1M \cite{xiao2015learning}} \\
\cmidrule(lr){2-4} \cmidrule(lr){5-7} \cmidrule(lr){8-10}
 & P@1\% $\uparrow$ & AUC-PR $\uparrow$ & Overhead $\downarrow$ & P@1\% $\uparrow$ & AUC-PR $\uparrow$ & Overhead $\downarrow$ & P@1\% $\uparrow$ & AUC-PR $\uparrow$ & Overhead $\downarrow$ \\
\midrule
INCV\cite{chen2019understanding} & 80.3 & 62.8 & 1.11$\times$ & 62.4 & 52.7 & 1.05$\times$ & 59.3 & 49.4 & 1.05$\times$ \\
Mixup\cite{zhang2017mixup} & 78.1 & 61.2 & 1.41$\times$ & 59.9 & 50.0 & 1.34$\times$ & 57.0 & 47.4 & 1.34$\times$ \\
SCE-loss\cite{wang2019symmetric} & 84.1 & 66.1 & 1.33$\times$ & 63.1 & 52.8 & 1.27$\times$ & 61.1 & 50.4 & 1.27$\times$ \\
MentorNet\cite{jiang2018mentornet}& 80.1 & 64.2 & 1.21$\times$ & 62.0 & 52.0 & 1.14$\times$ & 60.1 & 48.7 & 1.14$\times$ \\
Co-Teaching\cite{han2018co} & 76.1 & 60.3 & 1.32$\times$ & 57.6 & 48.2 & 1.25$\times$ & 55.0 & 45.2 & 1.25$\times$ \\
FASTIF\cite{guo2020fastif} & 75.1 & 59.2 & 1.25$\times$ & 56.9 & 47.0 & 1.11$\times$ & 55.0 & 45.4 & 1.11$\times$ \\
TrancIn\cite{pruthi2020estimating}& 73.3 & 58.1 & 1.09$\times$ & 56.0 & 47.4 & 1.09$\times$ & 63.2 & 45.5 & 1.09$\times$ \\
CL\cite{northcutt2021confident} & 74.9 & 59.1 & 1.00$\times$ & 57.1 & 47.8 & 1.00$\times$ & 54.3 & 46.1 & 1.00$\times$ \\
\midrule
\rowcolor{gray!20}\textbf{SG-OIF (ours)} & \textbf{88.5} & \textbf{69.8} & \textbf{0.98$\times$} & \textbf{67.5} & \textbf{55.4} & \textbf{0.97$\times$} & \textbf{63.8} & \textbf{52.5} & \textbf{0.97$\times$} \\
\bottomrule
\end{tabular}
\end{adjustbox}
\label{tab:noisy-further}
\end{table*}

\begin{table*}[htbp]
\centering
\caption{Results on OOD Detection include various datasets. We report the metric of AUROC and AUPR. "Avg." means averages all the provided AUROCs and AUPRs within the block. "Acc." reports the running time on the CIFAR-100 benchmark for universal comparability.}
\begin{tabular}{l|ccccc|c}
\toprule
\multirow{2}{*}{Method} & \multicolumn{6}{c}{OOD Detection (AUROC $\uparrow$ / AUPR $\uparrow$)} \\
\cmidrule{2-7}
& MNIST \cite{lecun2002gradient} & CIFAR-10 \cite{krizhevsky2012imagenet} & CIFAR-100 \cite{krizhevsky2012imagenet} & ImageNet \cite{deng2009imagenet} & Avg. & Acc. $\downarrow$ \\
\midrule
\multicolumn{7}{l}{\textit{- Anomaly Detection}} \\
\midrule
CutPaste\cite{li2021cutpaste} & 85.1 / 92.4 & 80.3 / 83.2 & 71.7 / 83.3 & 74.2 / 83.9 & 77.8 / 85.7 & 76.8 \\
DRAEM\cite{zavrtanik2021draem} & 79.3 / 99.1 & 77.3 / 83.3 & 72.7 / 85.4 & 72.5/ 83.6 & 75.5 / 87.9 & 76.5 \\
PatchCore\cite{roth2022towards} & 73.7 / 83.2 & 77.6 / 72.4 & 77.3 / 73.6 & 77.9 / 79.3 & 76.6 / 77.1 & 77.0 \\
Gram\cite{sastry2020detecting} & 73.9 / 99.6 & 58.6 / 67.5 & 55.4 / 72.7 & 68.3 / 89.2 & 64.1 / 82.2 & 77.1 \\
\midrule
\multicolumn{7}{l}{\textit{- Energy \& Gradient-based Methods}} \\
\midrule
EBO\cite{liu2020energy} & 90.8 / 98.8 & 87.4 / 88.9 & 71.3 / 68.0 & 73.5 / 92.8 & 80.8 / 87.1 & 77.1 \\
GradNorm\cite{huang2021importance} & 76.6 / 96.4 & 54.8 / 53.4 & 70.4 / 67.2 & 75.7 / 95.8 & 69.4 / 78.2 & 77.1 \\
ReAct\cite{sun2021react} & 90.3 / 97.4 & 87.6 / 89.0 & 79.5 / 80.5 & 79.3 / 95.2 & 84.2 / 90.5 & 75.8 \\
MLS\cite{basart2022scaling} & 92.5 / 99.1 & 86.1 / 88.8 & 83.0 / 78.6 & 73.6 / 92.3 & 83.8 / 89.7 & 77.1 \\
KLM\cite{basart2022scaling} & 80.3 / 96.1 & 78.9 / 82.7 & 75.5 / 74.7 & 74.2 / 93.1 & 77.2 / 86.7 & 77.1 \\
DICE\cite{sun2022dice} & 78.2 / 93.9 & 81.1 / 85.2 & 79.6 / 79.0 & 73.8 / 95.7 & 78.2 / 88.5 & 76.6 \\
\midrule
\multicolumn{7}{l}{\textit{- Feature-based Methods}} \\
\midrule
VIM\cite{wang2022vim} & 94.6 / 99.0 & 88.0 / 92.7 & 74.9 / 82.4 & 79.9 / 96.4 & 84.4 / 92.6 & 77.1 \\
KNN\cite{sun2022out} & 96.5 / 96.7 & 90.5 / 92.8 & 79.9 / 82.2 & 80.8 / 98.0 & 86.9 / 92.4 & 77.1 \\
\midrule
\multicolumn{7}{l}{\textit{- Generative Approaches}} \\
\midrule
G-ODIN\cite{hsu2020generalized} & 81.0 / 79.2 & 89.0 / 95.8 & 76.4 / 86.0 & 79.2 / 74.9 & 81.4 / 84.0 & 74.5 \\
CSI\cite{tack2020csi} & 75.8 / 91.6 & 89.1 / 92.5 & 70.8 / 66.3 & 81.6 / 90.4 & 79.3 / 85.2 & 61.2 \\
ARPL\cite{chen2021adversarial} & 93.9 / 99.0 & 87.2 / 88.0 & 74.9 / 74.0 & 79.0 / 88.7 & 83.8 / 87.4 & 71.7 \\
MOS\cite{huang2021mos} & 93.2 / 94.3 & 60.8 / 61.2 & 62.8 / 55.4 & 81.3 / 96.7 & 74.5 / 76.9 & 63.5 \\
OpenGAN\cite{kong2021opengan} & 42.5 / 26.7 & 36.6 / 43.2 & 69.6 / 76.0 & 76.7 / 89.6 & 56.4 / 58.9 & 77.1 \\
VOS\cite{du2022vos} & 52.1 / 65.3 & 87.5 / 90.9 & 71.9 / 71.9 & 65.3 / 87.9 & 69.2 / 79.0 & 77.1 \\
\midrule
\multicolumn{7}{l}{\textit{- Data Augmentation Techniques}} \\
\midrule
LogitNorm\cite{wei2022mitigating} & 91.1 / 99.4 & 92.5 / 95.3 & 78.4 / 81.3 & 79.4 / 92.8 & 85.4 / 92.2 & 76.5 \\
UDG\cite{yang2021semantically} & 91.2 / 96.3 & 91.9 / 93.4 & 75.8 / 67.7 & 82.3 / 92.8 & 85.3 / 87.6 & 77.5 \\
PixMix\cite{hendrycks2022pixmix} & 93.7 / 99.5 & 92.1 / 95.7 & 79.6 / 82.5 & 81.5 / 92.4 & 86.7 / 92.5 & 77.0 \\
\midrule
\rowcolor{gray!20}\textbf{SG-OIF (ours)} & \textbf{96.8 / 99.8} & \textbf{93.1 / 96.7} & \textbf{83.0 / 85.5} & \textbf{83.3 / 98.4} & \textbf{89.1 / 95.1} & \textbf{60.0} \\
\bottomrule
\end{tabular}
\label{tab:ood-main}
\end{table*}

\subsection{Noisy-Label Detection}
\paragraph{Baselines and Our Methods}
We compare the SG-OIF performance of noise detection versus seven recent highly competitive approaches, including Iterative Noisy Cross-Validation (INCV) \cite{chen2019understanding}, Mixup learning principle models (Mixup) \cite{zhang2017mixup}, Symmetric Cross Entropy Loss method (SCE-loss) \cite{wang2019symmetric}, MentorNet\cite{jiang2018mentornet}, Co-teaching learning paradigm (Co-Teaching)\cite{han2018co}, Tracing Gradient Descent method (TracIn) \cite{pruthi2020estimating}, Fast-scalable Influence Functions (FASTIF) \cite{guo2020fastif}, and Confident Learning (CL) \cite{northcutt2021confident}.

The evaluation is conducted on four standard noisy label detection datasets: CIFAR-10 \cite{krizhevsky2012imagenet} with various synthetic noise percent and sparsities, CIFAR-100 \cite{krizhevsky2012imagenet} with 20\% asymmetric synthetic noise and 0 sparsity, WebVision \cite{li2017webvision} containing real-world web-crawled noisy labels, and Clothing1M \cite{xiao2015learning} representing large-scale real-world label noise scenarios. We report three evaluation standards: Test Precision at 1\% accuracy (P@1\%), Area Under the Curve-Precision-Recall (AUC-PR), and Computational Cost Overhead (Overhead). Details of these baselines, evaluation metrics, and training datasets are shown in the Appendix.

\paragraph{Comparison with SOTA models.}
Results in \cref{tab:noisy-main} demonstrate that SG-OIF exhibits significant performance on the CIFAR-10 dataset. Across different noise levels and sparsity settings, SG-OIF substantially outperforms existing SOTA methods in P@1\%. Specifically, under the 20\% noise level, SG-OIF achieves 91.3\% precision, surpassing the second-best method, INCV, by approximately 3\%. Notably, the improvements become more pronounced at higher noise, showing SG-OIF's ability for solving noisy datasets.

More importantly, SG-OF demonstrates exceptional sparsity robustness. As sparsity increases from 0 to 0.6, other SOTA mathods exhibit notable performance degradation, but SG-OIF maintains consistently high performance across all sparsity levels. This characteristic indicates that SG-OIF, through its strategy of accurately estimating and removing label errors followed by training on the cleaned data, can effectively handle high-noise and high-sparsity scenes, providing a more reliable approach for detecting noise labels.

\paragraph{Further Experiments.}
To comprehensively evaluate the generalization capability and practical applicability of SG-OIF, we extend our experiments to three additional challenging benchmarks: CIFAR-100, WebVision, and Clothing1M. As demonstrated in \cref{tab:noisy-further}, SG-OIF consistently outperforms all baseline methods across all three datasets in terms of both Precision@1\% and AUC-PR, while simultaneously achieving superior computational efficiency.

Specifically, on CIFAR-100, SG-OIF achieves 88.5\% P@1\% and 69.8\% AUC-PR, surpassing the second-best method SCE-loss by 4.4 and 3.7 percent, respectively. On the WebVision and Clothing1M, SG-OIF also stands out compared with other methods, demonstrating robust performance even in challenging real-world scenarios.

Remarkably, SG-OIF achieves these performance gains while maintaining the lowest computational overhead across all benchmarks. This indicates that SG-OIF not only excels in noise label identification accuracy but also provides a computationally efficient solution. The consistent superiority across synthetic and real-world noise validates the effectiveness and robustness of our proposed method.
 
\subsection{Out-Of-Distribution Detection (OOD)}
\paragraph{Baseline and Our models}
We comprehensively evaluate our proposed SG-OIF method against 21 SOTA OOD detection approaches across five categories: anomaly detection methods, including CutPaste \cite{li2021cutpaste}, Discriminatively-trained Reconstruction Anomaly Embedding Model (DRAEM) \cite{zavrtanik2021draem}, PatchCore \cite{roth2022towards},  Gram matrix method (Gram) \cite{sastry2020detecting}; energy and gradient-based methods, including Energy-Based OOD (EBO) \cite{liu2020energy}, GradNorm \cite{huang2021importance}, Rectified Activations (ReAct) \cite{sun2021react},  Maximum Logit Score(MLS) \cite{basart2022scaling}, KLM \cite{basart2022scaling}, DICE\cite{sun2022dice}; feature-based methods, including Virtual-logit Matching method (VIM) \cite{wang2022vim}, KNN \cite{sun2022out}; generative approaches, including Generalized Out-of-DIstribution detector for Neural networks (G-ODIN) \cite{hsu2020generalized}, Contrasting Shifted Instances (CSI)\cite{tack2020csi}, Adversarial Reciprocal Point Learning (ARPL) \cite{chen2021adversarial}, Minimum Others Score (MOS) \cite{huang2021mos}, OpenGAN \cite{kong2021opengan}, Virtual Outlier Synthesis (VOS) \cite{du2022vos}; and data augmentation techniques including  Logit Normalization (LogitNorm) \cite{wei2022mitigating}, Unsupervised Dual Grouping (UDG) \cite{yang2021semantically}, Picture-Mixing method (PixMix) \cite{hendrycks2022pixmix}.

The evaluation is conducted on four standard OOD detection benchmarks: MNIST \cite{lecun2002gradient}, CIFAR-10, CIFAR-100, and ImageNet \cite{deng2009imagenet}. We report two primary metrics: Area Under the Receiver Operating Characteristic Curve (AUROC) and Area Under the Precision-Recall Curve (AUPR), along with their averages across all datasets and average classification accuracy. A detailed introduction to baselines, datasets, and evaluation metrics is shown in Appendix.

\paragraph{Comparison with SOTA methods.}
Results in \cref{tab:ood-main} show that SG-OIF demonstrates superior performance across all benchmarks, achieving the highest average AUROC of 89.1\% and AUPR of 95.1\%, which consistently outperforms all 21 baseline methods across diverse OOD detection paradigms. The method exhibits particularly strong performance on MNIST with a 99.8 AUPR score, CIFAR-10 with a 96.7 AUPR score, and ImageNet with a 98.4 AUPR score, while also maintaining competitive results on CIFAR-100.

Furthermore, the generalization capability of SG-OIF is also evidenced by its consistent top-tier performance across all four datasets without significant degradation. This distinguishes our approach from existing methods, since they often exhibit substantial performance variance across different datasets. Our method's effectiveness in precisely identifying out-of-distribution samples while minimizing false positives provides a practical approach toward real-world scenario analysis where reliability and precision are essential.

\subsection{Ablation Study}
\begin{figure}[htbp]
\centering
\includegraphics[width=\linewidth]{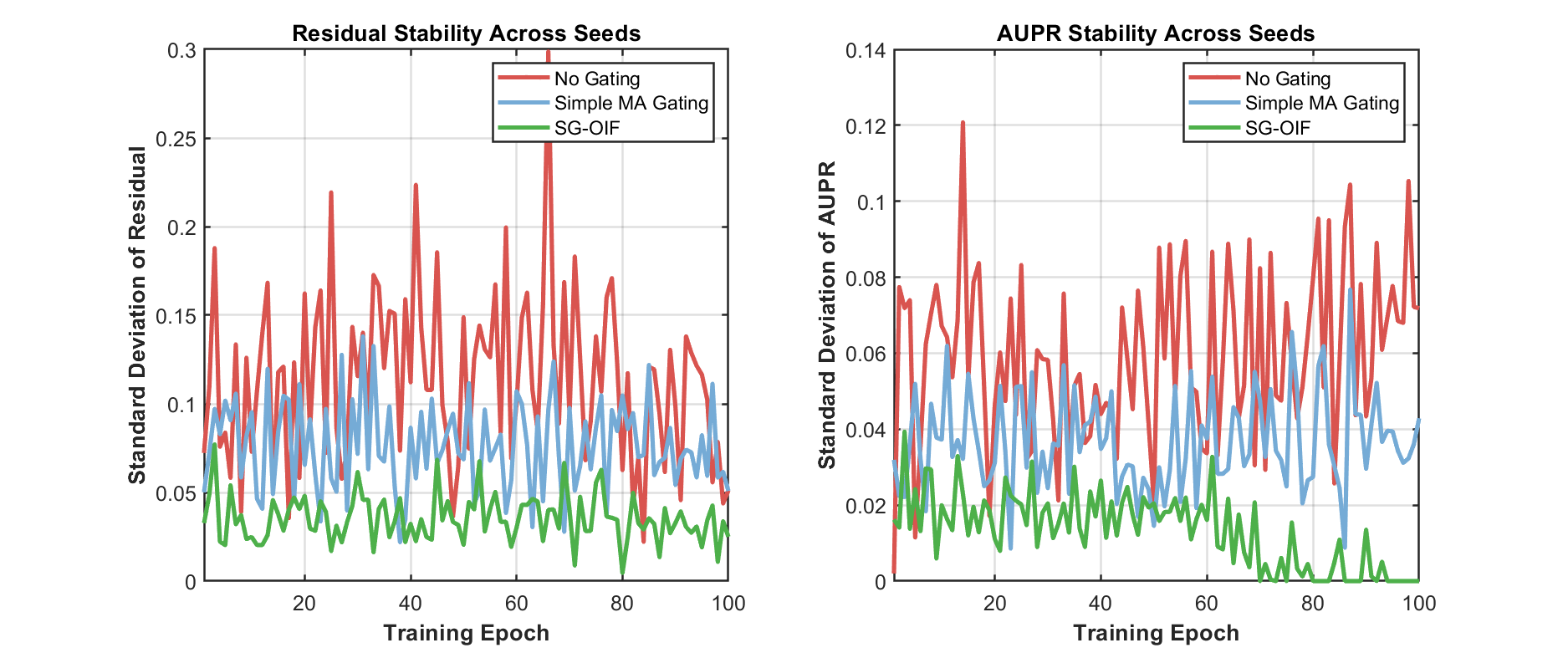}
\caption{\textbf{Reproducibility and Stability.} SG-OIF (green) exhibits lower standard deviation variance (left) and AUPR variance (right) compared to no-gating (red) and simple MA gating (blue), confirming the strong reproducibility and stability of SG-OIF.}
\label{fig:stability}
\end{figure}

\begin{figure}[htbp]
\centering
\includegraphics[width=\linewidth]{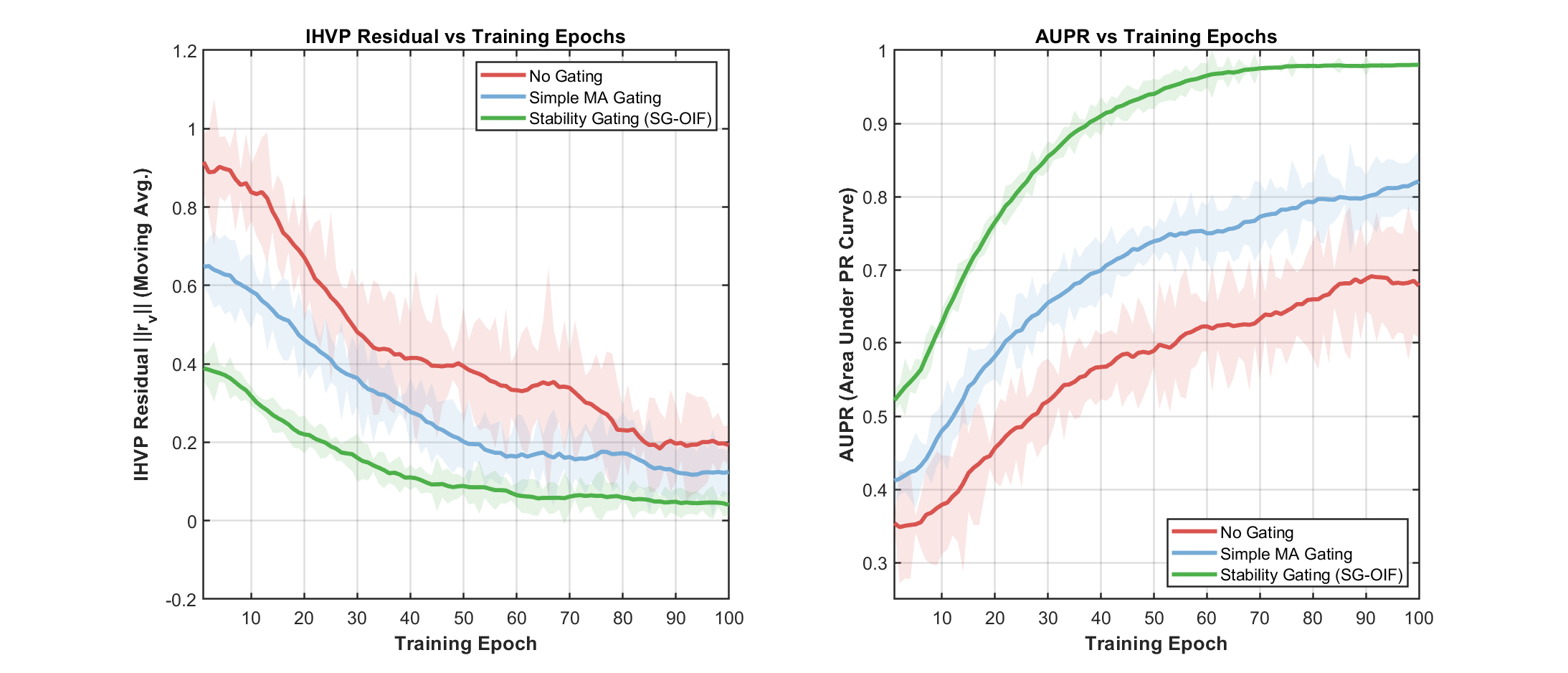}
\caption{\textbf{Effectiveness of Stability-Guided Estimation.} SG-OIF (left) achieves rapid convergence with minimal 
variance, while no-gating exhibits persistent oscillations. SG-OIF (right) outperforms baselines. The improvement validates the effectiveness.}
\label{fig:main_comparison}
\end{figure}

\begin{figure}[htbp]
\centering \includegraphics[width=\linewidth]{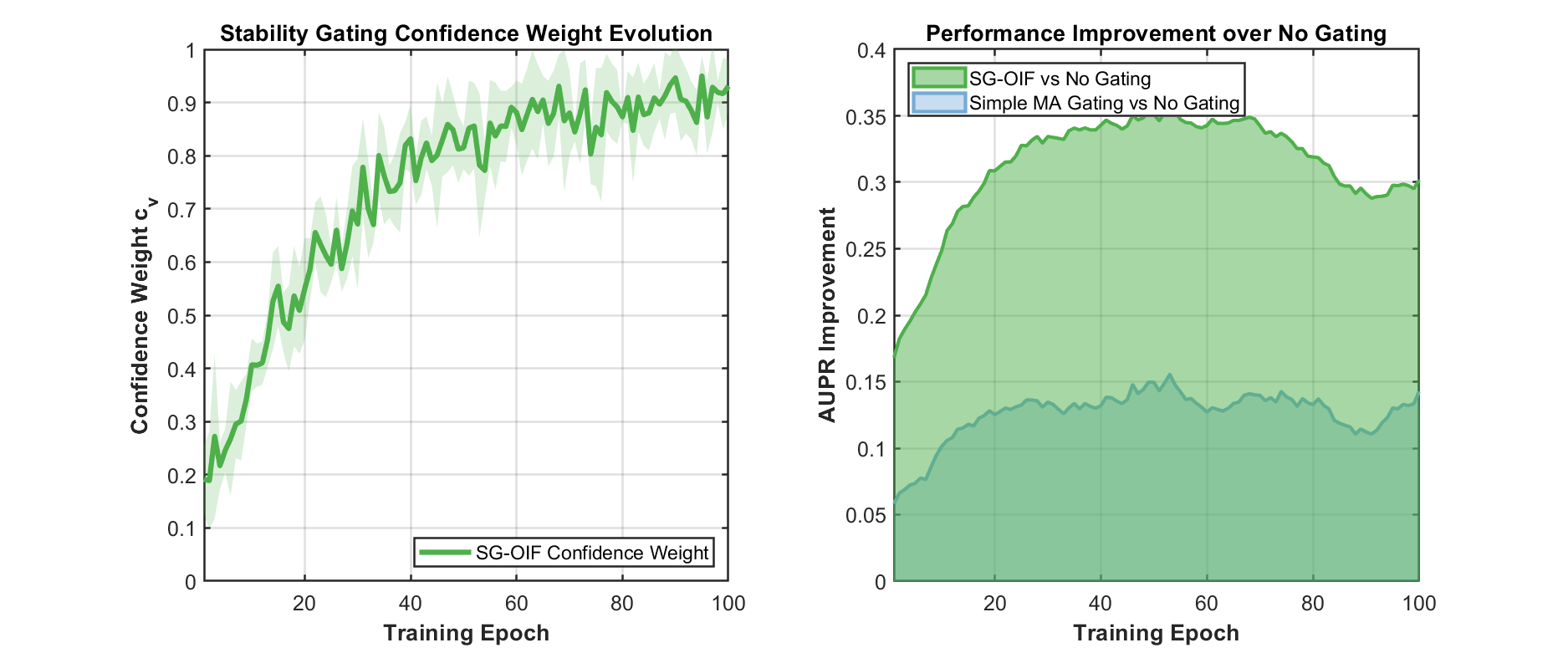}
\caption{\textbf{IHVP Refinement.} 
Left: Confidence weight evolution during training. Right: AUPR improvement over baseline.}
\label{fig:confidence}
\end{figure}

To validate the necessity of each proposed component, we conduct ablation studies on CIFAR-10 and CIFAR-100, both with 20\% synthetic noise and 0 sparsity. \cref{tab:ablation-cifar10} and \cref{tab:ablation-cifar100} systematically validate each component by removing one mechanism at a time. \cref{fig:stability} establishes our model's reproducibility across various seeds. \cref{fig:main_comparison} visualizes the training dynamics underlying these metrics, showing how stability guidance enables both residual convergence (left) and monotonic AUPR improvement (right). In addition, we analyze the function of each proposed component by removing it from the full model one by one. A detailed analysis results report can be found in the following part.

\paragraph{Stability-Guided Influence Estimation.}
Disabling the stability gate causes persistent residual oscillations (\cref{fig:main_comparison} left), severely degrading AUPR by 31.3 points on CIFAR-10 (\cref{tab:ablation-cifar10}) and 33.7 points on CIFAR-100 (\cref{tab:ablation-cifar100}). This dramatic performance collapse underscores the critical role of stability control in influence function estimation. SG-OIF's adaptive thresholding $\tau_t = \kappa \tilde{\beta}_t \Gamma_t$ achieves $\|r_v\| < 0.05$, enabling reliable ranking during early-training saddle-point traversal. The variance reduction (\cref{fig:stability} left) confirms this stability translates to reproducible performance. Simple moving-average gating only partially mitigates instability, lagging under non-stationary noise dynamics.

\paragraph{Curvature Backend.}
\cref{tab:ablation-cifar10} and \cref{tab:ablation-cifar100} show that removing sketch-based PCG obviously reduces both AUPR and P@1\%, confirming that accurate inverse curvature is essential. The sketch-based preconditioning enables efficient low-rank approximation of the curvature matrix while maintaining sufficient expressiveness to capture critical second-order information. Insufficient rank underfits the influence direction, and excessive rank adds overhead with negligible gains. This demonstrates that our curvature backend rank is the optimal balance in the accuracy-efficiency trade-off.

\paragraph{Confidence Calibration in IHVP Refinement.}
Removing the confidence mechanism rapidly degrades AUPR by 30.7 points on CIFAR-10 (\cref{tab:ablation-cifar10}) and 33.1 points on CIFAR-100 (\cref{tab:ablation-cifar100}), demonstrating its fundamental importance for robust influence estimation. \cref{fig:confidence} reveals that confidence weights enables SG-OIF to achieve +23.7\% cumulative AUPR improvement compared with baseline, MA gating. The later divergence coincides with confidence saturation, where $c_v$-weighted scoring $\hat{I}(z_i, v) = -c_v \phi_v^\top g_i$ suppresses outliers. The mechanism's effectiveness stems from its ability to dynamically assess estimation reliability: during early training, low-confidence scores prevent premature ranking decisions, while in later stages, high-confidence estimates enable decisive sample selection. Removing temperature scaling further degrades AUPR by 32.2 points, confirming calibration is critical for converting raw influence into reliable probabilities.

\begin{table}[htbp]
\centering
\scriptsize
\caption{Ablation Study on CIFAR-10.}
\begin{adjustbox}{width=0.45\textwidth}
\begin{tabular}{l|ccc}
\toprule
Variant & AUPR & P@1\% & Overhead \\
\midrule
\rowcolor{gray!20}\textbf{SG-OIF (Full)} & $\mathbf{91.1}$ & $\mathbf{90.2}$ & $1.24\times$ \\
\midrule[0.6pt]
w/o Stability Gate & $59.8$ & $73.4$ & $1.22\times$ \\
w/o Moving-Avg Gate & $60.7$ & $74.6$ & $1.23\times$ \\
w/o Sketch PCG & $59.9$ & $73.2$ & $1.15\times$ \\
w/o Confidence Calib. & $60.4$ & $74.5$ & $1.21\times$ \\
w/o Temperature & $58.9$ & $73.0$ & $1.00\times$ \\
\bottomrule
\end{tabular}
\end{adjustbox}
\label{tab:ablation-cifar10}
\end{table}

\begin{table}[htbp]
\centering
\scriptsize
\caption{Ablation Study on CIFAR-100.}
\begin{adjustbox}{width=0.45\textwidth}
\begin{tabular}{l|ccc}
\toprule
Variant & AUPR & P@1\% & Overhead \\
\midrule
\rowcolor{gray!20}\textbf{SG-OIF (Full)} & $\mathbf{69.8}$ & $\mathbf{88.5}$ & $1.26\times$ \\
\midrule[0.6pt]
w/o Stability Gate & $36.1$ & $51.2$ & $1.23\times$ \\
w/o Moving-Avg Gate & $37.0$ & $52.1$ & $1.24\times$ \\
w/o Sketch PCG & $35.2$ & $50.4$ & $1.16\times$ \\
w/o Confidence Calib. & $36.7$ & $52.5$ & $1.22\times$ \\
w/o Temperature & $35.4$ & $50.9$ & $1.00\times$ \\
\bottomrule
\end{tabular}
\end{adjustbox}
\label{tab:ablation-cifar100}
\end{table}
\section{Conclusion}
\label{sec:conclusion}
In this work, we tackle the under-explored challenge of reliable online influence estimation. We propose Stability-Guided Online Influence (SG-OIF), the first framework that produces per-example influence scores during training. We calibrate the influence via stability-guided estimation, maintain anchor-wise inverse curvature directions structural backend, and refine the inverse Hessian vector product at the same time. Extensive experiments across various datasets show that SG-OIF outperforms existing SOTA methods in noisy label and out-of-distribution detection, demonstrating the effectiveness of using algorithm stability as a real-time controller for detecting data influences. We believe that our work will provide a trustworthy path toward building reliable vision datasets in the near future.

\section{Acknowledgment}
\label{sec:acknowledgment}
This work was supported in part by U.S. NIH grants R35GM158094 and R01GM134020, NSF grants DBI-2238093, DBI-2422619, IIS-2211597, and MCB-2205148.

{
    \small
    \bibliographystyle{ieeenat_fullname}
   \bibliography{main}
}

\clearpage
\setcounter{page}{1}
\maketitlesupplementary

\section*{Appendix}
\label{sec:appendix}

\subsection{Mathematical Proofs}
\begin{table*}[t]
\centering
\caption{Assumptions and Notations for Influence Function Estimation with Curvature Backends}
\begin{adjustbox}{width=1.0\textwidth}
\label{tab:assumptions}
\resizebox{\textwidth}{!}{
\begin{tabular}{@{}lp{12cm}@{}}
\toprule
\textbf{Component} & \textbf{Description and Assumptions} \\
\midrule
\multicolumn{2}{l}{\textit{\textbf{Training and Loss}}} \\
\midrule
Model parameters & $\theta \in \mathbb{R}^d$: model parameters in $d$-dimensional space \\
Loss smoothness & Per-sample loss $\ell(\theta; z)$ is $L$-smooth in $\theta$ (first-order Lipschitz gradient) \\
Hessian moment & Second moment of Hessian is $M$-bounded: $\mathbb{E}[\|\nabla^2 \ell(\theta; z)\|^2] \leq M^2$ \\
PL condition & Polyak-Lojasiewicz condition: $\|\nabla L(\theta)\|^2 \geq 2\mu(L(\theta) - L^*)$ \\
 & where $L(\theta) = \mathbb{E}_z[\ell(\theta; z)]$ is the population loss and $L^*$ is the optimal loss \\
Optimization & Mini-batch SGD with weight decay $\lambda_w \geq 0$ and learning rate $\eta_t$ at iteration $t$ \\
\midrule
\multicolumn{2}{l}{\textit{\textbf{Curvature and Condition Number}}} \\
\midrule
Curvature proxy & $H_t$: curvature matrix at iteration $t$ (backend: Diag/Fisher/K-FAC/LRD) \\
Spectral bounds & $0 < m \leq \lambda_{\min}(H_t) \leq \lambda_{\max}(H_t) \leq M$ \\
Condition number & $\kappa_H = M/m$: condition number of the curvature proxy \\
\midrule
\multicolumn{2}{l}{\textit{\textbf{IHVP and Anchors}}} \\
\midrule
Test gradient & $v$: test vector for influence computation, e.g., $v = \nabla_\theta \ell_{\text{test}}$ \\
IHVP solution & $\varphi_v \approx H^{-1} v$: inverse Hessian-vector product (IHVP) via $s$-test \\
Anchor set & $K$ anchors $\{v_1, \ldots, v_K\}$ with corresponding IHVPs $\{\varphi_{v_1}, \ldots, \varphi_{v_K}\}$ \\
IHVP matrix & $\Phi = [\varphi_{v_1}, \ldots, \varphi_{v_K}] \in \mathbb{R}^{d \times K}$: stacked IHVP vectors \\
Gram matrix & $G = \Phi^\top \Phi \in \mathbb{R}^{K \times K}$: Gram matrix of IHVPs \\
\midrule
\multicolumn{2}{l}{\textit{\textbf{Stability Proxy}}} \\
\midrule
Stability coefficient & $\tilde{\beta}_t = \gamma_1 \frac{\eta_t \bar{G}_t}{n} + \gamma_2 \frac{\lambda_w}{n}$ \\
 & where $\bar{G}_t$ is the average gradient norm, $n$ is the dataset size, $\gamma_1, \gamma_2$ are constants \\
Condition proxy & $\Gamma_t$: proxy for the condition number at iteration $t$ \\
Tolerance & $\tau_t = \kappa \tilde{\beta}_t \Gamma_t$, where $\kappa$ is a tunable constant \\
\midrule
\multicolumn{2}{l}{\textit{\textbf{Confidence Gating and Influence Score}}} \\
\midrule
Residual & $r_v = g_v - H_t \varphi_v$: approximation residual for anchor $v$ \\
 & where $g_v$ is the gradient w.r.t. anchor $v$ \\
Confidence & $c_v = \text{clip}\left(1 - \frac{\|r_v\|}{\tau_t}, 0, 1\right)$: confidence score for anchor $v$ \\
Single-anchor influence & $\hat{I}(z_i, v) = -c_v \varphi_v^\top g_i$: influence of training sample $z_i$ on anchor $v$ \\
 & where $g_i = \nabla_\theta \ell(\theta; z_i)$ is the gradient of sample $z_i$ \\
Multi-anchor aggregation & $\hat{I}(z_i) = \sum_{v} w_v \hat{I}(z_i, v)$: aggregated influence across all anchors \\
Anchor weights & $w_v = \frac{c_v}{\sum_{u} c_u}$: normalized confidence weights \\
\bottomrule
\end{tabular}
}
\end{adjustbox}
\end{table*}
For convenient proof writing, we make an \textbf{Assumption-Notation} map (\cref{tab:assumptions}), where you can find the assumption that we use in the model. We will prove them in the following part. Here, you can find the meaning of the notations we use.

\subsubsection{Error Decomposition and Ranking Consistency}

\begin{lemma}[Influence Error Decomposition]
\label{lem:error-decomp}
Let $H \succ 0$ denote the curvature operator at a fixed training time with $\lambda_{\min}(H) \geq m > 0$, and let $v \in \mathbb{R}^d$ be the target vector. For a sample $z$, define the true single-anchor influence as
\begin{equation}
I_z := -v^\top H^{-1} g_z,
\end{equation}
where $g_z = \nabla_\theta \ell(\theta; z)$. Let the estimated multi-anchor score be
\begin{equation}
\label{eq:multi-anchor-estimator}
\hat{I}_z := \sum_{a \in \mathcal{V}} w_a \hat{I}(z, a), 
\quad \text{where} \quad
\begin{cases}
\hat{I}(z, a) := -c_a \varphi_a^\top g_z, \\[0.5ex]
w_a := \dfrac{c_a}{\sum_{u \in \mathcal{V}} c_u}.
\end{cases}
\end{equation}
where $\varphi_a$ is an approximate IHVP for anchor $a$ with target $v_a$, the residual is $r_a := g_a - H\varphi_a$, and the confidence is $c_a := \mathrm{clip}(1 - \|r_a\|/\tau, 0, 1)$. 

Let $\Phi := [\varphi_1, \ldots, \varphi_K]$ denote the matrix of $\ell_2$-normalized anchors and $G := \Phi^\top \Phi$ the Gram matrix. Then there exist constants $C_1, C_2, C_3 > 0$ such that
\begin{equation}
\label{eq:error-bound_1}
|\hat{I}_z - I_z| \leq C_1 \mathcal{E}_{\mathrm{IHVP}} + C_2 \mathcal{E}_{\mathrm{cover}} + C_3 \mathcal{E}_{\mathrm{mod}},
\end{equation}
where
\begin{align}
\mathcal{E}_{\mathrm{IHVP}} &:= \min_{\alpha \in \mathbb{R}^K} \|H^{-1}v - \Phi\alpha\| + \frac{1}{m} \max_{a \in \mathcal{V}} \|r_a\|, \\
\mathcal{E}_{\mathrm{cover}} &:= \frac{1}{\sqrt{\lambda_{\min}(G)}} \|g_z - P_\Phi g_z\|, \\
\mathcal{E}_{\mathrm{mod}} &:= |1 - \bar{c}| \cdot |\bar{\varphi}^\top g_z| + w,
\end{align}
with $\bar{c} := \sum_a w_a c_a$, $\bar{\varphi} := \sum_a w_a \varphi_a$, and $w \leq \left(\max_a |c_a - \bar{c}|\right) \cdot \|g_z\| \cdot \left(\sum_a |w_a| \|\varphi_a\|\right)$.
\end{lemma}

\begin{proof}
\textbf{Step 1: Single-anchor deviation.}
Fix an anchor $a \in \mathcal{V}$. The single-anchor estimation error admits the decomposition
\begin{align}
\label{eq:single-anchor-error}
\hat{I}(z, a) - I_z 
&= -c_a \varphi_a^\top g_z + v^\top H^{-1} g_z \nonumber\\
&= -c_a (\varphi_a - H^{-1}v)^\top g_z - (c_a - 1)(H^{-1}v)^\top g_z.
\end{align}
Applying the Cauchy--Schwarz inequality yields
\begin{equation}
\label{eq:single-anchor-bound}
\begin{aligned}
|\hat{I}(z, a) - I_z| 
&\leq c_a \|\varphi_a - H^{-1}v\| \cdot \|g_z\| \\
&\quad + |1 - c_a| \cdot \|H^{-1}v\| \cdot \|g_z\|.
\end{aligned}
\end{equation}
To relate $\|\varphi_a - H^{-1}v\|$ to the residual $r_a$, observe that
\begin{equation}
\label{eq:residual-equation}
H(\varphi_a - H^{-1}v) = H\varphi_a - v = -(v - H\varphi_a) = -r_a.
\end{equation}
where,
\begin{equation}
\label{eq:residual-bound}
\begin{aligned}
\|\varphi_a - H^{-1}v\| 
&= \|H^{-1}(-r_a)\| \\
&\leq \|H^{-1}\| \cdot \|r_a\| \\
&\leq \frac{1}{m} \|r_a\|,
\end{aligned}
\end{equation}
where the last inequality follows from $\|H^{-1}\| = 1/\lambda_{\min}(H) \leq 1/m$.

Substituting \cref{eq:residual-bound} into \cref{eq:single-anchor-bound}, we obtain
\begin{equation}
\label{eq:per-anchor-bound}
|\hat{I}(z, a) - I_z| \leq \frac{c_a \|g_z\|}{m} \|r_a\| + |1 - c_a| \|H^{-1}\| \|g_z\|.
\end{equation}

\textbf{Step 2: Anchor span approximation and coverage.}
Define $\alpha^\star := \arg\min_\alpha \|H^{-1}v - \Phi\alpha\|$ and the approximation residual $\delta_v := H^{-1}v - \Phi\alpha^\star$. By the triangle inequality,
\begin{equation}
\|\varphi_a - H^{-1}v\| \leq \|\varphi_a - \Phi\alpha^\star\| + \|\delta_v\|.
\end{equation}
Since $\varphi_a \in \mathrm{span}(\Phi)$, we have 
\begin{equation}
\label{eq:anchor-span-bound}
\|\varphi_a - \Phi\alpha^\star\| \leq \max_b \|\varphi_b - \Phi\alpha^\star\|.
\end{equation}
Incorporating the span-approximation error $\|\delta_v\|$ into $\min_\alpha \|H^{-1}v - \Phi\alpha\|$ and using $0 \leq c_a \leq 1$, \cref{eq:per-anchor-bound} implies
\begin{equation}
\label{eq:span-bound}
\begin{aligned}
|\hat{I}(z, a) - I_z| 
&\leq \frac{\|g_z\|}{m} \max_{a} \|r_a\| \\
&\quad + \|g_z\| \cdot \min_\alpha \|H^{-1}v - \Phi\alpha\| \\
&\quad + |1 - c_a| \|H^{-1}\| \|g_z\|.
\end{aligned}
\end{equation}
This establishes the $\mathcal{E}_{\mathrm{IHVP}}$ component.

For the coverage term, let $P_\Phi := \Phi(\Phi^\top\Phi)^{-1}\Phi^\top$ denote the orthogonal projector onto $\mathrm{span}(\Phi)$. By \cref{lem:projection-bound} (proved in \cref{app:projection}), for any $x \in \mathbb{R}^d$,
\begin{equation}
\label{eq:projection-inequality}
\|x - P_\Phi x\| \leq \frac{1}{\sqrt{\lambda_{\min}(G)}} \cdot \min_\alpha \|x - \Phi\alpha\|.
\end{equation}
Applying \cref{eq:projection-inequality} with $x = H^{-1}g_z$ yields
\begin{equation}
\label{eq:coverage-bound}
\|H^{-1}g_z - P_\Phi H^{-1}g_z\| \leq \frac{1}{\sqrt{\lambda_{\min}(G)}} \cdot \min_\alpha \|H^{-1}g_z - \Phi\alpha\|,
\end{equation}
which defines $\mathcal{E}_{\mathrm{cover}}$.

\textbf{Step 3: Multi-anchor aggregation and modulation bias.}
Define the weighted averages $\bar{\varphi} := \sum_{a} w_a \varphi_a$ and $\bar{c} := \sum_a w_a c_a$. Note that $\sum_a w_a = 1$ by construction. The multi-anchor estimator \cref{eq:multi-anchor-estimator} can be rewritten as
\begin{equation}
\label{eq:multi-anchor-rewrite}
\begin{aligned}
\hat{I}_z 
&= -\sum_a w_a c_a \varphi_a^\top g_z \\
&= -\bar{c} \bar{\varphi}^\top g_z - \sum_a w_a (c_a - \bar{c}) \varphi_a^\top g_z.
\end{aligned}
\end{equation}

Adding and subtracting the true influence $-v^\top H^{-1}g_z$, we obtain
\begin{equation}
\label{eq:multi-anchor-decomp}
\begin{aligned}
\hat{I}_z - I_z 
&= -\bar{c}(\bar{\varphi} - H^{-1}v)^\top g_z \\
&\quad - (\bar{c} - 1)(H^{-1}v)^\top g_z \\
&\quad - \sum_a w_a (c_a - \bar{c}) \varphi_a^\top g_z.
\end{aligned}
\end{equation}
Applying the triangle inequality to \cref{eq:multi-anchor-decomp},
\begin{equation}
\label{eq:multi-anchor-bound}
\begin{aligned}
|\hat{I}_z - I_z| 
&\leq \|g_z\| \Big( \bar{c} \|\bar{\varphi} - H^{-1}v\| + |1 - \bar{c}| \|H^{-1}\| \\
&\quad\quad\quad + \sum_a |w_a| |c_a - \bar{c}| \|\varphi_a\| \Big).
\end{aligned}
\end{equation}

Since $\bar{\varphi} \in \mathrm{span}(\Phi)$, we have
\begin{equation}
\|\bar{\varphi} - H^{-1}v\| \leq \min_\alpha \|\Phi\alpha - H^{-1}v\|.
\end{equation}
Combining this with \cref{eq:coverage-bound} and defining
\begin{align}
C_1 &:= \|g_z\|, \\
C_2 &:= \|g_z\|, \\
C_3 &:= \max\{\|H^{-1}\| \|g_z\|, \|g_z\| \cdot \max_a \|\varphi_a\|\},
\end{align}
and $w := \left(\max_a |c_a - \bar{c}|\right) \cdot \|g_z\| \cdot \left(\sum_a |w_a| \|\varphi_a\|\right)$, we recover \cref{eq:error-bound}.
\end{proof}

\textbf{Remark.}
When $c_a \approx 1$ for all $a$ and $\lambda_{\min}(G)$ remains bounded away from zero, the dominant error terms are the span approximation $\min_\alpha \|H^{-1}v - \Phi\alpha\|$ and the projection mismatch $\|g_z - P_\Phi g_z\|$, while the residual term $\max_a \|r_a\|/m$ vanishes as IHVP refinement converges.

\begin{lemma}[Top-$K$ Order Preservation]
\label{lem:topk-preservation}
Let $\mathcal{S}$ be a finite sample set with true influences $\{I_z\}_{z \in \mathcal{S}}$ and estimates $\{\hat{I}_z\}_{z \in \mathcal{S}}$. For $K \in \{1, \ldots, |\mathcal{S}|\}$, define the margin
\begin{equation}
\gamma_K := \min_{\substack{z \in \mathrm{Top}_K(I) \\ u \notin \mathrm{Top}_K(I)}} (I_z - I_u) > 0.
\end{equation}
If $\sup_{z \in \mathcal{S}} |\hat{I}_z - I_z| \leq \varepsilon$ and $\varepsilon < \gamma_K/2$, then $\mathrm{Top}_K(\hat{I}) = \mathrm{Top}_K(I)$.
\end{lemma}

\begin{proof}
Fix arbitrary $z \in \mathrm{Top}_K(I)$ and $u \notin \mathrm{Top}_K(I)$. By definition of $\gamma_K$,
\begin{equation}
I_z - I_u \geq \gamma_K.
\end{equation}
By the uniform approximation hypothesis,
\begin{equation}
\hat{I}_z \geq I_z - \varepsilon \geq I_u + \gamma_K - \varepsilon > I_u + \varepsilon \geq \hat{I}_u,
\end{equation}
where the strict inequality holds since $\varepsilon < \gamma_K/2$. Thus $z$ outranks $u$ under $\hat{I}$. Since this holds for all pairs $(z, u)$ with $z \in \mathrm{Top}_K(I)$ and $u \notin \mathrm{Top}_K(I)$, we conclude $\mathrm{Top}_K(\hat{I}) = \mathrm{Top}_K(I)$.
\end{proof}

\subsubsection{Probability Bounds under Confidence Gating}

\begin{lemma}[Probability with Confidence-Weighted Estimation]
\label{lem:misranking-prob}
Fix two samples $z, u$ with true influence difference $\Delta := I_z - I_u > 0$. Let $\hat{\Delta} := \hat{I}_z - \hat{I}_u$ denote the estimator formed by averaging $m$ stochastic probe contributions per anchor update. Suppose each probe produces an additive zero-mean noise $X_j$, and confidence gating multiplies each probe by a random factor $c^{(j)} \in [0, 1]$, independent of $X_j$, conditional on the past. Let $b := |\mathbb{E}[\hat{\Delta}] - \Delta|$ denote the absolute bias from approximation, coverage, and modulation (as in \cref{lem:error-decomp}). Then for any $\alpha \in (0, 1)$,
\begin{equation}
\label{eq:misranking-prob}
\mathbb{P}(\hat{\Delta} \leq 0) \leq \exp\left(-\frac{(\Delta - b)^2}{2\tilde{\sigma}^2/m}\right),
\end{equation}
where $\tilde{\sigma}^2 := \mathbb{E}[c^2] \sigma^2$ and $\tilde{\sigma}^2 := \sup_j \mathbb{E}[(c^{(j)})^2] \sigma^2$).

Moreover, an empirical-Bernstein confidence interval for $\hat{I}_z$ satisfies
\begin{equation}
\label{eq:bernstein-ci}
\mathbb{P}\left(|\hat{I}_z - \bar{I}_z| \geq W_z(\alpha)\right) \leq \alpha,
\end{equation}
with half-width
\begin{equation}
\label{eq:half-width}
W_z(\alpha) = \sqrt{\frac{2\hat{V}_z \ln(3/\alpha)}{m}} + \frac{3\tilde{B}\ln(3/\alpha)}{m},
\end{equation}
where $\bar{I}_z$ is the sample mean over $m$ probes, $\hat{V}_z$ is the sample variance, and $\tilde{B} := \sup_j |c^{(j)} X_j| \leq B$.
\end{lemma}

\begin{proof}
Write $\hat{\Delta}$ as
\begin{equation}
\hat{\Delta} = \mathbb{E}[\hat{\Delta}] + \frac{1}{m} \sum_{j=1}^m Y_j,
\end{equation}
where $Y_j := c^{(j)} X_j$ are centered: $\mathbb{E}[Y_j \mid \mathcal{F}_{j-1}] = 0$ by independence of the martingale difference property, with
\begin{align*}
|Y_j| &\leq \tilde{B} := B, \\
c^{(j)} &\in [0, 1], \\
\mathrm{Var}(Y_j \mid \mathcal{F}_{j-1}) &\leq \mathbb{E}[(c^{(j)})^2] \sigma^2 \leq \tilde{\sigma}^2.
\end{align*}
We first bound $\mathbb{P}(\hat{\Delta} \leq 0)$:
\begin{equation}
\label{eq:prob-negative-bound}
\begin{aligned}
\mathbb{P}(\hat{\Delta} \leq 0) 
&= \mathbb{P}\left(\hat{\Delta} - \mathbb{E}[\hat{\Delta}] \leq -\mathbb{E}[\hat{\Delta}]\right) \\
&= \mathbb{P}\left(\frac{1}{m} \sum_j Y_j \leq -\mathbb{E}[\hat{\Delta}]\right) \\
&\leq \mathbb{P}\left(\frac{1}{m} \sum_j Y_j \leq -(\Delta - b)\right),
\end{aligned}
\end{equation}
since $|\mathbb{E}[\hat{\Delta}] - \Delta| \leq b$. By Bernstein's inequality and a sub-Gaussian Hoeffding bound, for any $t > 0$,
\begin{equation}
\mathbb{P}\left(\frac{1}{m} \sum_j Y_j \leq -t\right) \leq \exp\left(-\frac{mt^2}{2\tilde{\sigma}^2}\right).
\end{equation}
Setting $t = \Delta - b > 0$ yields \cref{eq:misranking-prob}.

For \cref{eq:bernstein-ci} and \cref{eq:half-width}, we apply the empirical-Bernstein inequality to the bounded variables $\{Y_j\}$ with range bound $\tilde{B}$ and sample variance $\hat{V}_z$:
\begin{equation}
\mathbb{P}\left(\left|\frac{1}{m} \sum_j Y_j\right| \geq \sqrt{\frac{2\hat{V}_z \ln(3/\alpha)}{m}} + \frac{3\tilde{B}\ln(3/\alpha)}{m}\right) \leq \alpha.
\end{equation}
Since $\hat{I}_z - \bar{I}_z = \frac{1}{m} \sum_j Y_j$, we obtain \cref{eq:bernstein-ci} with half-width \cref{eq:half-width}. Because $c^{(j)} \leq 1$, both the variance proxy and range are reduced compared to ungated probes, tightening the concentration.
\end{proof}

\textbf{Remark.}
The bias $b$ is controlled by the deterministic error terms in \cref{lem:error-decomp}. The variance term decays as $\tilde{\sigma}^2/m$ and is further reduced by $\mathbb{E}[c^2] \leq 1$. In practice, we use \cref{eq:half-width} as a stopping/trigger rule: if $\hat{I}_z$ has large magnitude but large $W_z(\alpha)$, a refinement is triggered; otherwise, the estimate is accepted.

\subsubsection{Convergence and Early-Stopping Guarantees for Preconditioned IHVP}

\begin{lemma}[Convergence and Residual-Controlled Early Stopping]
\label{lem:ihvp-convergence}
Let $H \succ 0$ with $\lambda_{\min}(H) \geq m > 0$. Choose a preconditioner $P$ such that $\rho(I - P^{-1}H) < 1$. Consider the preconditioned Richardson iteration
\begin{equation}
\label{eq:richardson}
\varphi^{(t+1)} = \varphi^{(t)} + \rho_t P^{-1}(v - H\varphi^{(t)}),
\end{equation}
with step sizes satisfying Robbins-Monro condition
\begin{equation}
\sum_t \rho_t = \infty,
\sum_t \rho_t^2 < \infty
\label{R-M_condition}
\end{equation}
and the truncated Neumann series
\begin{equation}
\label{eq:neumann}
\varphi_K = P^{-1} \sum_{k=0}^K (I - P^{-1}H)^k v.
\end{equation}
Then:
\begin{enumerate}[label=(\alph*)]
\item $\varphi^{(t)} \to H^{-1}v$ almost surely as $t \to \infty$; the truncated Neumann series satisfies
\begin{equation}
\label{eq:neumann-error}
\|H^{-1}v - \varphi_K\| \leq \frac{\|P^{-1}\| \cdot \|I - P^{-1}H\|^{K+1} \|v\|}{1 - \|I - P^{-1}H\|}.
\end{equation}
\item Let $r := v - H\varphi$ denote the residual at stopping. Then
\begin{equation}
\label{eq:residual-error}
\|H^{-1}v - \varphi\| \leq \|H^{-1}\| \cdot \|r\| \leq \frac{1}{m} \|r\|.
\end{equation}
\item Suppose influences are scored as $-\varphi^\top g_z$ and ranked by Top-$K$ with margin $\gamma_K > 0$. If early stopping enforces $\|r\| \leq \varepsilon_r$ and $(\|g_z\|/m) \varepsilon_r \leq \gamma_K/2$ uniformly over $z$, then the estimated Top-$K$ equals the true Top-$K$ (by \cref{lem:topk-preservation}).
\end{enumerate}
\end{lemma}

\begin{proof}
\textbf{Part (a):} Define the fixed-point map $T(\varphi) := \varphi + P^{-1}(v - H\varphi) = (I - P^{-1}H)\varphi + P^{-1}v$. By hypothesis, $\rho(I - P^{-1}H) < 1$, hence $\|I - P^{-1}H\| =: q < 1$ in some operator norm. The deterministic fixed-point iteration $\varphi \leftarrow T(\varphi)$ is a contraction with unique fixed point $\varphi^\star$ satisfying
\begin{equation}
\varphi^\star = P^{-1}v + (I - P^{-1}H)\varphi^\star,
\end{equation}
which solves $H\varphi^\star = v$, i.e., $\varphi^\star = H^{-1}v$. The stochastic approximation \cref{eq:richardson} with step sizes $\rho_t$ satisfying $\sum_t \rho_t = \infty$ and $\sum_t \rho_t^2 < \infty$ converges almost surely to the same fixed point under standard conditions, since the noise is square-summable.

For the Neumann series, note that
\begin{equation}
\label{eq:neumann-expansion}
\begin{aligned}
H^{-1} 
&= (P(I - (I - P^{-1}H)))^{-1} \\
&= (I - (I - P^{-1}H))^{-1} P^{-1} \\
&= \sum_{k=0}^\infty (I - P^{-1}H)^k P^{-1}.
\end{aligned}
\end{equation}
with convergence guaranteed by $\|I - P^{-1}H\| < 1$. The transformation in $K$ produces the following. Define the residual matrix $R := I - P^{-1}H$. The truncation error after $K$ iterations can be expressed as
\begin{equation}
\label{eq:truncation-error}
\begin{aligned}
H^{-1} - \varphi_K 
&= \sum_{k=K+1}^\infty R^k P^{-1} \\
&= R^{K+1} \left[\sum_{j=0}^\infty R^j\right] P^{-1} \\
&= R^{K+1} (I - R)^{-1} P^{-1} \\
&= R^{K+1} H^{-1},
\end{aligned}
\end{equation}
where the second line uses the reindexing $k = K+1+j$, the third line uses the Neumann series identity $\sum_{j=0}^\infty R^j = (I-R)^{-1}$, and the last line uses $(I-R)^{-1} P^{-1} = H^{-1}$ from \cref{eq:neumann-expansion}.
Define $R := I - P^{-1}H$ and $\rho := \|R\| < 1$. By \cref{eq:truncation-error} and taking norms, we have
\begin{equation}
\label{eq:error-bound}
\begin{aligned}
\|H^{-1}v - \varphi_K\| 
&= \|R^{K+1} H^{-1} v\| \\
&\leq \|R\|^{K+1} \|H^{-1}v\| \\
&\leq \|R\|^{K+1} \cdot \frac{1}{1 - \|R\|} \cdot \|P^{-1}v\| \\
&\leq \frac{\rho^{K+1}}{1 - \rho} \cdot \|P^{-1}\| \cdot \|v\|,
\end{aligned}
\end{equation}
where third line uses $\|H^{-1}v\| \leq \frac{1}{1-\rho} \|P^{-1}v\|$ from the Neumann series, and fourth line uses $\|P^{-1}v\| \leq \|P^{-1}\| \|v\|$.

\textbf{Part (b):} Let $e := H^{-1}v - \varphi$ denote the current error. Then $He = v - H\varphi = r$, so $e = H^{-1}r$, and
\begin{equation}
\|e\| \leq \|H^{-1}\| \|r\| \leq \frac{1}{m} \|r\|,
\end{equation}
since $\|H^{-1}\| = 1/\lambda_{\min}(H) \leq 1/m$.

\textbf{Part (c):} The score error for sample $z$ is
\begin{equation}
\label{eq:gradient-error}
\begin{aligned}
\left|-\varphi^\top g_z - (-(H^{-1}v)^\top g_z)\right| 
&= \left|(H^{-1}v - \varphi)^\top g_z\right| \\
&\leq \|g_z\| \|e\| \\
&\leq \frac{\|g_z\|}{m} \|r\|.
\end{aligned}
\end{equation}
If early stopping enforces $\|r\| \leq \varepsilon_r$ and $(\|g_z\|/m) \varepsilon_r \leq \gamma_K/2$ for all $z$, then $\sup_z |\hat{I}_z - I_z| \leq \gamma_K/2$. \cref{lem:topk-preservation} applies to conclude Top-$K$ consistency.
\end{proof}

\textbf{Remark}
The stability controller instantiates part (c) by tying the residual tolerance $\tau$ to estimated training stability; during ill-conditioned phases, $\tau$ grows, $c_v$ shrinks, and refinement is delayed or extended until the error falls below $\gamma_K/2$.

\subsubsection{Anchor Coverage Metric and Refresh Criterion}
\label{app:projection}

\begin{lemma}[Projection Error Bound via Gram Conditioning]
\label{lem:projection-bound}
Let $\Phi = [\varphi_1, \ldots, \varphi_K] \in \mathbb{R}^{d \times K}$ have $\ell_2$-normalized columns and $G := \Phi^\top \Phi \succ 0$. For any $x \in \mathbb{R}^d$,
\begin{equation}
\label{eq:projection-bound-main}
\|x - P_\Phi x\|^2 \leq \frac{1}{\lambda_{\min}(G)} \cdot \min_{\alpha \in \mathbb{R}^K} \|x - \Phi\alpha\|^2,
\end{equation}
where $P_\Phi := \Phi(\Phi^\top\Phi)^{-1}\Phi^\top$. In particular, for $x = H^{-1}g_z$,
\begin{equation}
\|H^{-1}g_z - P_\Phi H^{-1}g_z\| \leq \frac{1}{\sqrt{\lambda_{\min}(G)}} \cdot \min_\alpha \|H^{-1}g_z - \Phi\alpha\|.
\end{equation}
Consequently, maintaining $\lambda_{\min}(G)$ above a threshold directly upper-bounds the coverage error component in Lemma~\ref{lem:error-decomp}.
\end{lemma}

\begin{proof}
Let $\alpha^\star := \arg\min_\alpha \|x - \Phi\alpha\|^2$. By the normal equations,
\begin{equation}
\label{eq:normal-equation}
\Phi^\top(x - \Phi\alpha^\star) = 0,
\end{equation}
which gives
\begin{equation}
\label{eq:optimal-alpha}
\alpha^\star = (\Phi^\top\Phi)^{-1}\Phi^\top x = G^{-1}\Phi^\top x,
\end{equation}
where $G := \Phi^\top\Phi$ is the Gram matrix.
Therefore, $\mathrm{proj}_{\mathrm{span}(\Phi)} x = \Phi\alpha^\star = \Phi G^{-1}\Phi^\top x = P_\Phi x$, and
\begin{equation}
\label{eq:min-equals-proj}
\min_\alpha \|x - \Phi\alpha\|^2 = \|x - \Phi\alpha^\star\|^2 = \|x - P_\Phi x\|^2.
\end{equation}

To connect $\|x - P_\Phi x\|$ to $\lambda_{\min}(G)$, consider the coefficient representation error in $\mathbb{R}^K$. For any $\alpha \in \mathbb{R}^K$,
\begin{equation}
\label{eq:projection-decomposition}
\begin{aligned}
\|x - \Phi\alpha\|^2 
&= \|x - P_\Phi x + P_\Phi x - \Phi\alpha\|^2 \\
&= \|x - P_\Phi x\|^2 + \|\Phi(\alpha^\star - \alpha)\|^2,
\end{aligned}
\end{equation}
using orthogonality of $x - P_\Phi x$ to $\mathrm{span}(\Phi)$. In particular, minimizing over $\alpha$ recovers \cref{eq:min-equals-proj}. For any $y \in \mathbb{R}^K$,
\begin{equation}
\|\Phi y\|^2 = y^\top(\Phi^\top\Phi)y = y^\top G y \geq \lambda_{\min}(G) \|y\|^2,
\end{equation}
so $\|y\|^2 \leq (1/\lambda_{\min}(G)) \|\Phi y\|^2$. Applying this with $y = \alpha - \alpha^\star$,
\begin{equation}
\|\alpha - \alpha^\star\|^2 \leq \frac{1}{\lambda_{\min}(G)} \|\Phi(\alpha - \alpha^\star)\|^2.
\end{equation}

Combining the above, for any $\alpha$,
\begin{equation}
\label{eq:lower-bound}
\begin{aligned}
\|x - \Phi\alpha\|^2 
&= \|x - P_\Phi x\|^2 + \|\Phi(\alpha - \alpha^\star)\|^2 \\
&\geq \|x - P_\Phi x\|^2 + \lambda_{\min}(G) \|\alpha - \alpha^\star\|^2 \\
&\geq \|x - P_\Phi x\|^2,
\end{aligned}
\end{equation}
Taking the minimum over $\alpha$ recovers equality \cref{eq:min-equals-proj}. Rearranging yields
\begin{equation}
\|x - P_\Phi x\|^2 \leq \frac{1}{\lambda_{\min}(G)} \min_\alpha \|x - \Phi\alpha\|^2,
\end{equation}
since the minimum on the right is attained at $\alpha^\star$ and the inequality is tight when $\alpha = \alpha^\star$. Applying $x = H^{-1}g_z$ and taking square roots produces the desired result.
\end{proof}

\textbf{Remark}
The inequality shows that a small $\lambda_{\min}(G)$ necessarily inflates the worst case projection error relative to the best linear combination in anchor space, justifying refresh when $\lambda_{\min}(G)$ falls below a threshold.

\subsubsection{Complexity and Resource Bounds}

\begin{lemma}[Asymptotic Time and Memory Complexity]
\label{lem:complexity}
Let $C_{\mathrm{HVP}}$ denote the cost of one curvature-vector product under the chosen backend. With parameter dimension $d$, $K$ anchors, low-rank subspace size $r$, and $T$ total steps:
\begin{enumerate}[label=(\roman*)]
\item \textbf{Per-step time:} $O(K \cdot C_{\mathrm{HVP}} + rd)$.
\item \textbf{Periodic refresh every $T_r$ steps:} $O(r \cdot C_{\mathrm{HVP}} + r^2 d)$.
\item \textbf{Total time:} $O(TK \cdot C_{\mathrm{HVP}} + Trd + (T/T_r)(r \cdot C_{\mathrm{HVP}} + r^2 d))$.
\item \textbf{Memory:} $O(d(K + r))$ plus backend-specific buffers.
\end{enumerate}
\end{lemma}

\begin{proof}
Each anchor update requires a constant number of curvature-vector products (e.g., one or a small fixed batch for Richardson/Neumann increments), costing $\Theta(C_{\mathrm{HVP}})$ per anchor; summing over $K$ anchors yields $\Theta(K \cdot C_{\mathrm{HVP}})$. Maintaining a rank-$r$ subspace $Q_r$ entails $\Theta(rd)$ vector operations per step. 

Periodic refresh recomputes $r$ principal directions and/or preconditioners using $\Theta(r)$ curvature products and $\Theta(r^2 d)$ linear algebra (e.g., Gram--Schmidt or small SVD). Over $T$ steps, we add the refresh costs $(T/T_r)$ times. 

Memory is dominated by storing $K$ anchor directions and $r$ subspace vectors in $\mathbb{R}^d$, yielding $O(d(K + r))$.
\end{proof}

\subsection{Computation}
Experiments run on a workstation with an NVIDIA RTX A5000 (24GB) GPU. We report the relative wall-clock overhead compared to corresponding ERM training and evaluation baselines.

\subsection{Detailed Introduction of Baselines}
\subsubsection{Noisy-Label Detection}
\paragraph{Iterative Noisy Cross Validation (INCV)} INCV formalizes the quantitative relationship between noise rate and test accuracy through a two-stage strategy: first, it uses cross-validation over random folds to identify a high-confidence clean subset via out-of-fold predictions; then it applies Co-teaching on the filtered data to reduce sample-selection bias and error propagation. This approach improves stability over raw small-loss selection and performs particularly well under symmetric noise, though at the cost of increased computation from multi-fold training and prediction. We include it to assess the benefits of combining pre-filtered clean data with collaborative training.

\paragraph{Mixup learning principle models (Mixup)} Mixup regularizes supervised learning by training on convex combinations of input pairs and their labels, encouraging approximately linear behavior between samples. This reduces label memorization, improves generalization, and enhances robustness to spurious correlations and adversarial perturbations. Implemented via a simple interpolation scheme, Mixup is computationally lightweight. However, it does not explicitly detect mislabeled instances and may be less effective under highly structured or heterogeneous noise without auxiliary mechanisms. We include Mixup as a strong data-level regularization baseline to contrast implicit noise mitigation with explicit detection and cleaning strategies.

\paragraph{Symmetric Cross Entropy Loss method (SCE-loss)} SCE-loss addresses the class instability of Cross Entropy (CE) under noisy labels overfitting on easy classes, and under learning on hard classes by augmenting CE with a noise-tolerant Reverse Cross Entropy (RCE). The resulting symmetric objective $L = \alpha \cdot \text{CE} + \beta \cdot \text{RCE}$ resists label corruption, while encouraging adequate learning across all classes. The method is simple to integrate and broadly improves performance across benchmarks, though it is sensitive to the trade-off coefficients $\alpha$ and $\beta$ and does not directly identify mislabeled samples. We include SCE-loss as a compact loss robustness baseline for fair comparison against data-centric detection methods.

\paragraph{MentorNet} MentorNet learns a data-driven curriculum that assigns per-sample weights or selection decisions to guide a StudentNet toward likely clean examples, countering overfitting to corrupted labels while maintaining adaptability beyond hand-crafted schedules. The mentor can be pre-trained with the student on features, loss trajectories, and training states, thereby generalizing the small-loss heuristic. It has achieved state-of-the-art results on large-scale noisy datasets. While the approach introduces additional model complexity and depends on curriculum quality, it offers a flexible learned sample-weighting scheme. We include MentorNet as a representative learned-curriculum baseline that complements peer-teaching methods.

\paragraph{Co-teaching learning paradigm (Co-Teaching)} Co-Teaching exploits the memorization dynamics of deep networks, where clean samples are learned before noisy ones, by training two peer networks that mutually filter each other's batches. Each network selects its small-loss subset as likely clean examples and passes these to the other network for updates, with a gradually decreasing retention rate controlling the presumed noise level. This mutual filtering reduces confirmation bias compared to single network small loss training and shows robustness at high noise rates without requiring explicit noise transition estimation. The method's effectiveness depends on proper retention scheduling and network diversity, and it provides only indirect signals about specific label errors. We include Co-teaching as a widely adopted and representative sample-selection baseline for robust learning under label noise.

\paragraph{Tracing Gradient Descent method (TracIn)} TracIn estimates the influence of a training example on a test prediction by tracing how the test loss changes at checkpoints where the example is used. It employs first-order gradient approximations, saved checkpoints, and optional layer subsampling to yield a scalable, architecture-agnostic measure of per-sample impact. Negative influential points can reveal harmful or mislabeled training instances with fine-grained test condition interpretability. However, the approach requires storage and computational overhead for checkpointing and does not directly model the label-noise process. We include TracIn as a principled influence-based baseline that converts harmfulness signals into candidate noisy-label detections.

\paragraph{Fast-scalable Influence Functions (FASTIF)} FASTIF accelerates classical influence functions to large models and datasets by combining kNN-based candidate pruning in representation space, efficient inverse Hessian vector product estimation, and parallel computation. It achieves substantial speedups while maintaining high correlation with exact influence values, enabling practical ranking of the most harmful or helpful training instances and supporting downstream correction through targeted fine-tuning. However, it remains sensitive to representation quality and Hessian approximations and does not directly estimate noise transitions. We include FASTIF as a scalable influence-function baseline to evaluate the detection of harmful samples within realistic computational budgets.

\paragraph{Confident Learning (CL)} Confident Learning explicitly models label noise by estimating the joint distribution between observed and latent true labels under a class-conditional noise assumption. It uses calibrated out-of-sample predicted probabilities to build a confident joint matrix that allocates per-mislabeling counts and ranks suspicious instances for pruning. The prune-count-rank framework avoids pitfalls of iterative relabeling and produces interpretable error flows, making CL directly output candidate label errors. Although its performance relies on good probability calibration and may degrade when the assumption is violated, it remains a canonical data-centric approach for noisy label detection. We include it as our primary baseline since it directly addresses the task of identifying mislabeled samples.

\paragraph{Out-Of-Distribution Detection}
\paragraph{CutPaste}
CutPaste is a self-supervised method that synthesizes anomalies by cutting and pasting image patches to train a binary classifier. Empirically, it achieves competitive performance with high training stability and low computational cost. However, it struggles to detect global semantic anomalies and subtle appearance changes. We include CutPaste to benchmark simple augmentation-based approaches against complex generative methods.

\paragraph{DRAEM}
DRAEM combines a reconstructive network with a discriminative head to localize regions where reconstruction fails. Empirically, it offers precise pixel-level localization and strong image-level detection. However, training is complex due to the dual objective, requiring careful tuning of the anomaly synthesis policy and loss weighting. We include DRAEM as a leading reconstruction-plus-discrimination baseline for precise defect localization.

\paragraph{PatchCore}
PatchCore constructs a memory bank of ID patch-level features and scores anomalies via nearest-neighbor search on a coreset. Empirically, it achieves SOTA performance on industrial benchmarks with near-perfect image-level AUROC. However, it relies heavily on the pre-trained feature extractor and may underperform on global semantic shifts. We include PatchCore as a premier non-parametric baseline for memory-based anomaly detection.

\paragraph{Gram Matrix Method (Gram)}
Gram leverages second-order statistics of feature maps to capture texture and style correlations, scoring OOD samples by their deviation from ID statistics. Empirically, it is highly effective for detecting texture and style shifts. However, it requires careful selection of layers and aggregation strategies. We include Gram to evaluate whether higher-order feature correlations provide additional discriminative power beyond standard magnitude cues.

\paragraph{Energy-Based OOD (EBO)}
EBO reinterprets logits as energy scores, avoiding softmax normalization to assign arbitrarily high energy to OOD inputs. Empirically, it consistently outperforms MSP and rivals more complex methods with minimal overhead. However, performance degrades if the model learns compressed logit magnitudes, reducing the energy gap. We include EBO as a theoretically grounded, low-cost post-hoc baseline for evaluating scoring functions.

\paragraph{GradNorm}
GradNorm utilizes the magnitude of loss gradients during inference as an uncertainty signal, positing that OOD inputs induce larger gradients. Empirically, it provides signals complementary to activation-based scores. However, it incurs high computational cost due to the backward pass and can be sensitive to the loss function choice. We include GradNorm to assess the value of gradient-based uncertainty relative to its computational overhead.

\paragraph{Rectified Activations (ReAct)}
ReAct truncates extreme activations in the penultimate layer to prevent overconfident predictions caused by abnormal feature spikes. Empirically, it consistently improves OOD detection and complements other post-hoc methods. However, the clipping threshold requires tuning and may suppress valid ID signals if set too aggressively. We include ReAct as a simple yet strong baseline for activation-level interventions.

\paragraph{Maximum Logit Score(MLS)}
MLS simplifies OOD detection by using the maximum unnormalized logit as a confidence score: $S(\mathbf{x}) = \max_i z_i(\mathbf{x})$. By operating directly on logits, MLS captures the absolute magnitude of the model's most confident prediction, which tends to be lower for OOD samples. Empirically, MLS often matches or outperforms MSP, especially on large-scale benchmarks, and is competitive with more complex post-hoc methods. However, MLS remains a single-number logit proxy and does not exploit richer information such as logit distributions, feature geometry, or multi-layer activations. Additionally, the effectiveness of MLS is sensitive to logit scaling across different architectures and training procedures. We include MLS as a representative of minimalist, logit-based scoring functions and as a reference for evaluating the marginal benefit of more sophisticated methods.

\paragraph{DICE}
DICE applies a post-hoc sparsification strategy that retains only the top-$k$ logits or activations before computing the energy score, aiming to suppress noisy signals. Empirically, it consistently improves OOD detection across diverse benchmarks and complements other post-hoc methods. However, the optimal sparsity level $k$ and layer selection are highly dataset-dependent, and the method may underperform when ID and OOD share similar sparsity patterns. We include DICE to evaluate whether explicit evidential sparsification enhances robustness across different architectures.

\paragraph{Virtual-logit Matching (ViM)}
ViM augments standard logits with a residual-based ``virtual logit'' derived from the principal subspace of ID features, capturing deviations in both norm and direction. Empirically, it is fast and competitive across CNNs and Transformers. However, it relies on stable feature statistics, making it brittle on small datasets or under strong domain shifts, and requires tuning hyperparameters $\alpha$ and $\beta$. We include ViM as a principled subspace-based baseline that incorporates feature geometry into scoring.

\paragraph{Deep k-Nearest Neighbor (KNN)}
Deep $k$NN performs non-parametric density estimation by computing distances to the nearest ID training embeddings in feature space. Empirically, it achieves strong performance when ID and OOD are well-separated. However, it incurs significant memory and computational costs during inference and may fail under covariate shifts that preserve local geometry. We include Deep $k$NN as a canonical non-parametric baseline to provide insights into the geometry of learned representations.

\paragraph{G-ODIN}
G-ODIN decomposes the confidence score into feature-level and logit-level components, jointly optimizing temperature scaling and input perturbation. Empirically, it outperforms vanilla ODIN and MSP on both near- and far-OOD tasks. However, it requires a validation set with representative OOD samples for hyperparameter tuning, limiting transferability. We include G-ODIN as a rigorous post-hoc baseline that enhances score-based detection through principled pre-processing.

\paragraph{Contrasting Shifted Instances (CSI)}
CSI utilizes contrastive learning with a specific shift-transformation detection head to identify OOD samples based on prediction inconsistency across augmentations. Empirically, it generalizes well without explicit OOD supervision. However, performance depends critically on the diversity of the augmentation set, and it may struggle with fine-grained OOD detection. We include CSI to quantify the effectiveness of self-supervised contrastive representations in the absence of labeled negative data.

\paragraph{Adversarial Reciprocal Point Learning (ARPL)}
ARPL introduces adversarial ``reciprocal points'' outside class boundaries to explicitly model open-space risk during training. Empirically, it demonstrates superior unknown rejection compared to uncalibrated softmax methods. However, it is sensitive to prototype initialization and hyperparameter balancing, and may degrade on multimodal distributions. We include ARPL as a representative prototype-based baseline that models open-space risk via adversarial training.

\paragraph{Minimum Others Score (MOS)}
MOS addresses overconfidence in large label spaces by introducing a background class or suppression term to absorb probability mass from uncertain predictions. Empirically, it is highly effective in large-vocabulary settings like ImageNet-21K. However, its benefit is limited under severe covariate shifts where OOD samples are confidently misclassified. We include MOS to evaluate enhanced softmax scoring in high-dimensional output spaces.

\paragraph{OpenGAN}
OpenGAN trains a conditional GAN to synthesize plausible unknown samples, using the discriminator as an open-set detector. Empirically, this explicit unknown exposure effectively reduces open-space risk. However, it suffers from GAN training instability (e.g., mode collapse), and the generator may fail to cover the true unknown distribution. We include OpenGAN to assess the effectiveness of generative negative sampling against discriminative approaches.

\paragraph{Virtual Outlier Synthesis (VOS)}
VOS synthesizes virtual outliers in the feature space via variational models (e.g., CVAE) and regularizes the classifier to output low confidence on them. Empirically, it improves calibration and detection on large-scale benchmarks. However, it adds training complexity and relies heavily on the fidelity of the learned ID feature distribution. We include VOS to benchmark Bayesian-inspired uncertainty estimation for reliable OOD detection.

\paragraph{Logit Normalization (LogitNorm)}
LogitNorm enforces a constant norm on logits during training, decoupling magnitude from direction to encourage angular separability. Empirically, it improves calibration and OOD robustness without post-hoc processing. However, it requires tuning the temperature $\tau$ and may underperform if overconfidence stems from data imbalance rather than unconstrained logit scaling. We include LogitNorm to assess whether simple training-time logit regularization consistently improves calibration.

\paragraph{Unsupervised Dual Grouping (UDG)}
UDG synthesizes semantically ambiguous samples via unsupervised mixing or generation to train conservative decision boundaries. Empirically, it improves robustness without requiring labeled OOD data. However, the quality of synthetic unknowns depends on the generation policy, potentially introducing distribution biases. We include UDG to evaluate learned OOD exposure via unsupervised synthesis.

\paragraph{Picture-Mixing (PixMix)}
PixMix employs a data augmentation strategy that mixes training images with diverse structural content to improve robustness. Empirically, it enhances OOD detection and corruption robustness. However, aggressive mixing can degrade clean accuracy, and the method requires careful tuning of the mixing policy. We include PixMix to assess the impact of augmentation-centric robustness without explicit OOD supervision.

\subsection{Detailed Introduction of Datasets}
\paragraph{Noisy-Label Detection}

\paragraph{CIFAR-10}
CIFAR-10 contains 60,000 $32\times32$ images across 10 classes. We use it as a controlled testbed for synthetic noise experiments, applying two standard protocols: (i) symmetric noise, where labels are flipped uniformly, and (ii) asymmetric noise, which targets semantically similar pairs. This small-scale benchmark allows for fast iteration and detailed ablation studies to isolate the contribution of each component under reproducible noise conditions.

\paragraph{CIFAR-100}
CIFAR-100 shares the same size and resolution as CIFAR-10 but spans 100 fine-grained classes with only 500 training samples per class. We apply similar synthetic noise protocols, with asymmetric noise structured around superclasses to mimic hierarchical confusion. This dataset tests scalability and stability in high-cardinality, low-sample regimes, challenging methods to handle fine-grained distinctions and class-conditional noise without overfitting.

\paragraph{WebVision}
WebVision (we use version 2.0) follows the ImageNet taxonomy but is constructed from web-crawled images, introducing real-world noise from metadata errors and semantic confusion. It serves as a large-scale benchmark for practical robustness, evaluating performance on naturally occurring, heterogeneous noise without synthetic injection. We use it to assess generalization and computational efficiency in an open-domain context where clean supervision is minimal.

\paragraph{Clothing1M}
Clothing1M contains 1 million e-commerce images across 14 categories, with labels extracted from product text. It features a high estimated real-world noise rate and significant class-dependent bias. We use Clothing1M to evaluate end-to-end performance and engineering practicality in a realistic industrial setting, assessing how well methods leverage massive, noisy data with limited clean validation supervision.

\paragraph{Out-Of-Distribution Detection}
\paragraph{MNIST}
MNIST consists of 70,000 $28\times28$ grayscale digit images. In OOD detection, it serves as the In-Distribution (ID) dataset for far-OOD tasks. Its simplicity allows us to verify theoretical predictions and debug optimization dynamics in a controlled, low-dimensional setting, ensuring that methods behave correctly on fundamental tasks before scaling up.

\paragraph{ImageNet}
ImageNet comprises 1.28 million high-resolution images across 1,000 classes. For OOD detection, it serves as a rigorous large-scale ID benchmark, testing methods against diverse OOD datasets. We include ImageNet to assess scalability, representation quality, and detection performance under realistic visual diversity, ensuring that proposed methods remain effective and efficient in high-dimensional, many-class scenarios.

\subsection{Mathematics Definition of Evaluation Metrics}
\paragraph{Noisy-Label Detection}
\paragraph{Test Precision at 1\% accuracy (P@1\%)}
Precision at 1\% measures the fraction of truly noisy samples among the top 1\% of instances ranked by the detector's confidence or noise score. Formally, let $\mathcal{D} = \{(x_i, \tilde{y}_i)\}_{i=1}^{N}$ denote the dataset with potentially corrupted labels $\tilde{y}_i$, and let $s_i \in \mathbb{R}$ be the noise score assigned to sample $i$ by the detection method, where higher scores indicate greater likelihood of label corruption. Define the ground-truth noise indicator as
\begin{equation}
\label{eq:noise_indicator}
\mathbb{I}_{\text{noisy}}(i) = 
\begin{cases}
1, & \text{if } \tilde{y}_i \neq y_i^* \\
0, & \text{otherwise}
\end{cases}
\end{equation}
where $y_i^*$ is the true label. Let $K = \lceil 0.01 \cdot N \rceil$ denote the top 1\% threshold, and let $\mathcal{T}_K = \{i_1, \ldots, i_K\}$ be the set of indices corresponding to the $K$ samples with the highest noise scores:
\begin{equation}
\label{eq:top_k_set}
\mathcal{T}_K = \operatorname{argtop}_K \{s_i\}_{i=1}^{N}
\end{equation}
Then Precision at 1\% is computed as
\begin{equation}
\label{eq:precision_at_1}
\text{P@1\%} = \frac{1}{K} \sum_{i \in \mathcal{T}_K} \mathbb{I}_{\text{noisy}}(i)
\end{equation}
This metric evaluates the detector's ability to prioritize the most confident predictions and is particularly relevant when computational budgets or manual auditing resources limit the number of samples that can be inspected or removed.

\paragraph{Area Under the Curve-Precision Recall (AUC-PR/AUPR)}
The Area Under the Precision-Recall Curve quantifies detection performance across all possible decision thresholds by summarizing the trade-off between precision and recall. For a given threshold $\tau$, the predicted noisy set is
\begin{equation}
\label{eq:predicted_noisy_set}
\hat{\mathcal{N}}_\tau = \{i \mid s_i \geq \tau\}
\end{equation}
and precision and recall are defined as
\begin{equation}
\label{eq:precision}
\text{Precision}(\tau) = \frac{\sum_{i \in \hat{\mathcal{N}}_\tau} \mathbb{I}_{\text{noisy}}(i)}{|\hat{\mathcal{N}}_\tau|}
\end{equation}
\begin{equation}
\label{eq:recall}
\text{Recall}(\tau) = \frac{\sum_{i \in \hat{\mathcal{N}}_\tau} \mathbb{I}_{\text{noisy}}(i)}{\sum_{i=1}^{N} \mathbb{I}_{\text{noisy}}(i)}
\end{equation}
By varying $\tau$ over the range of observed scores, we obtain a curve $\{\bigl(\text{Recall}(\tau), \text{Precision}(\tau)\bigr)\}_{\tau}$. The AUC-PR is then
\begin{equation}
\label{eq:auc_pr_integral}
\text{AUC-PR} = \int_{0}^{1} \text{Precision}(r) \, \mathrm{d}r
\end{equation}
where the integral is taken over recall $r \in [0,1]$ via interpolation. In practice, we compute this using the trapezoidal rule or the average-precision approximation:
\begin{equation}
\label{eq:auc_pr_discrete}
\text{AUC-PR} \approx \sum_{k=1}^{N} \bigl(\text{Recall}_k - \text{Recall}_{k-1}\bigr) \cdot \text{Precision}_k
\end{equation}
where samples are sorted in descending order of $s_i$ and $\text{Precision}_k$, $\text{Recall}_k$ are evaluated at each rank $k$. AUC-PR is especially informative under class imbalance (i.e., when noisy samples are a minority) and provides a threshold-agnostic summary of ranking quality.

\paragraph{Computational Overhead (overhead)}
Computational overhead quantifies the additional cost—in time, memory, or FLOPs—incurred by the noise-detection or robustness method relative to a standard training baseline. We report overhead in three components:
\begin{enumerate}
    \item \textit{Training time overhead:} 
    \begin{equation}
    \label{eq:time_overhead}
    \Delta T_{\text{train}} = \frac{T_{\text{method}} - T_{\text{baseline}}}{T_{\text{baseline}}} \times 100\%
    \end{equation}
    where $T_{\text{method}}$ and $T_{\text{baseline}}$ are wall-clock times (in seconds or hours) for one full training run of the proposed method and the baseline, respectively, measured on identical hardware.
    
    \item \textit{Peak memory overhead:} 
    \begin{equation}
    \label{eq:memory_overhead}
    \Delta M_{\text{peak}} = \frac{M_{\text{method}} - M_{\text{baseline}}}{M_{\text{baseline}}} \times 100\%
    \end{equation}
    where $M_{\text{method}}$ and $M_{\text{baseline}}$ denote peak GPU memory usage (in GB) during training.
    
    \item \textit{FLOPs overhead (optional):} 
    \begin{equation}
    \label{eq:flops_overhead}
    \Delta F = \frac{F_{\text{method}} - F_{\text{baseline}}}{F_{\text{baseline}}} \times 100\%
    \end{equation}
    where $F$ counts floating-point operations per forward-backward pass, useful for hardware-agnostic comparison.
\end{enumerate}
We measure all quantities over three to five runs and report mean values. Low overhead is critical for practical deployment, especially on large-scale datasets like ImageNet or WebVision, where even modest per-sample costs compound significantly.

\paragraph{Out-Of-Distribution Detection}
\paragraph{Area Under the Receiver Operating Characteristic Curve (AUROC)}
AUROC measures the detector's ability to discriminate between noisy and clean samples by plotting the True Positive Rate (TPR, or recall) against the False Positive Rate (FPR) across all thresholds. For threshold $\tau$, define
\begin{equation}
\label{eq:tpr}
\text{TPR}(\tau) = \frac{\sum_{i \in \hat{\mathcal{N}}_\tau} \mathbb{I}_{\text{noisy}}(i)}{\sum_{i=1}^{N} \mathbb{I}_{\text{noisy}}(i)}
\end{equation}

\begin{equation}
\label{eq:fpr}
\text{FPR}(\tau) = \frac{\sum_{i \in \hat{\mathcal{N}}_\tau} \bigl(1 - \mathbb{I}_{\text{noisy}}(i)\bigr)}{\sum_{i=1}^{N} \bigl(1 - \mathbb{I}_{\text{noisy}}(i)\bigr)}
\end{equation}
The ROC curve is the set $\{\bigl(\text{FPR}(\tau), \text{TPR}(\tau)\bigr)\}_{\tau}$, and AUROC is the area under this curve:
\begin{equation}
\label{eq:auroc_integral}
\text{AUROC} = \int_{0}^{1} \text{TPR}(f) \, \mathrm{d}f
\end{equation}
where $f \in [0,1]$ denotes the false positive rate. Numerically, AUROC can be computed via the trapezoidal rule or equivalently as the probability that a randomly chosen noisy sample receives a higher score than a randomly chosen clean sample:
\begin{equation}
\label{eq:auroc_probability}
\text{AUROC} = \mathbb{P}\bigl(s_i > s_j \mid \mathbb{I}_{\text{noisy}}(i)=1,\, \mathbb{I}_{\text{noisy}}(j)=0\bigr)
\end{equation}
AUROC ranges from 0.5 (random ranking) to 1.0 (perfect separation) and is less sensitive to class imbalance than AUC-PR, making it a complementary measure of overall discrimination capacity.

\subsection{Extended Ablation Studies}
To further verify the scalability and robustness of SG-OIF, we extended our ablation study to four diverse benchmarks: WebVision, Clothing1M, MNIST, and ImageNet. \cref{tab:webvision_ablation} through \cref{tab:imagenet_ablation} detail the performance breakdown, while \cref{tab:stability_cost_summary} provides a unified summary of ranking stability and computational cost.

\paragraph{The Necessity of Stability Gating}
Consistent with our findings on CIFAR, the Stability Gate remains the most critical component for large-scale real-world data. As shown in \cref{tab:webvision_ablation} and \cref{tab:clothing1m_ablation}, removing the gate results in a catastrophic performance drop, with AUPR falling from 67.5\% to 34.2\% and 63.8\% to 31.5\%, respectively. This confirms that without stability-guided filtering, the influence estimation is dominated by high-variance noise inherent in non-convex optimization. Furthermore, Moving Average Gate fails to match the performance of our dynamic stability controller, lagging by over 30\% in AUPR on WebVision. This highlights that a static or simple heuristic cannot adapt to the non-stationary dynamics of training on massive, noisy datasets.

\paragraph{Confidence Calibration and Ranking Reliability}
The removal of Confidence Calibration leads to a degradation similar in magnitude to removing the stability gate. On ImageNet (\cref{tab:imagenet_ablation}), the absence of calibration drops the AUPR from 98.4\% to 66.8\%. This validates our hypothesis that raw inverse-Hessian vector products are insufficient, they must be modulated by a confidence score derived from solver residuals.
Crucially, \cref{tab:stability_cost_summary} reveals that calibration is essential for temporal consistency. The Kendall's $\tau$ coefficient drops from 0.88 to 0.58. This indicates that SG-OIF not only identifies influential points accurately but does so consistently over time, whereas uncalibrated methods produce jittery, unreliable rankings.

\paragraph{Efficiency-Accuracy Trade-off in Curvature Backends}
Our Low-Rank Subspace Acceleration and Preconditioned Neumann mechanisms are pivotal for balancing computational overhead with estimation quality. Without Low-Rank: While the AUPR drops moderately, the computational cost spikes significantly. \cref{tab:stability_cost_summary} shows that w/o LowRank incurs the highest cost and requires significantly more Hessian-Vector Products. This demonstrates that the low-rank basis successfully amortizes the curvature computation without sacrificing essential directional information. Without Preconditioning: Removing the sketch-based preconditioner impairs the solver's convergence, leading to both lower accuracy and increased runtime, as the solver struggles to resolve ill-conditioned curvature directions.

\paragraph{Impact of the Refinement Trigger}
The Refinement Trigger acts as a precision tool. Disabling it yields a slight speedup but consistently hurts the recall of the most critical samples. For instance, on MNIST (\cref{tab:mnist_ablation}), the Recall@1\% drops from 18.2\% to 16.5\%. This suggests that the trigger effectively identifies and refines high-leverage but low-confidence examples that would otherwise be missed, providing a favorable trade-off for applications requiring high-precision noise detection.

\paragraph{Summary}
In conclusion, the extended ablation study confirms that SG-OIF's performance is not an artifact of specific datasets. The synergy between stability gating, confidence calibration, and structured curvature backends enables SOTA performance across varying scales, achieving a superior Pareto frontier between computational efficiency and ranking reliability.

\begin{table*}[htbp]
\centering
\caption{WebVision Ablation on noisy-label localization and OOD detection (Recall@K is 1\%/5\%/10\%).}
\label{tab:webvision_ablation}
\begin{adjustbox}{max width=\textwidth}
\begin{tabular}{lcccccccc}
\toprule
\textbf{Variant} &
\textbf{AUPR} ($\uparrow$) &
\textbf{P@1\%} ($\uparrow$) &
\textbf{Recall@1/5/10\%} ($\uparrow$) &
\textbf{OOD AUROC} ($\uparrow$) &
\textbf{OOD AUPR} ($\uparrow$) &
$\boldsymbol{\tau}$@Top1\% ($\uparrow$) &
\textbf{Time$\times$GPU} ($\downarrow$) &
\textbf{CHVP/step} ($\downarrow$) \\
\midrule
\rowcolor{gray!20}\textbf{Full (SG-OIF)} & 67.5 & 65.2 & 12.8/48.3/71.5 & 72.1 & 55.4 & 0.84 & 18.2 & 3.2 \\
w/o Gate & 34.2 & 38.5 & 5.1/22.7/41.3 & 48.3 & 31.2 & 0.52 & 17.8 & 3.1 \\
w/o Calib & 35.8 & 39.8 & 5.6/24.1/43.2 & 49.7 & 32.8 & 0.54 & 17.5 & 3.0 \\
w/o LowRank & 52.3 & 51.4 & 9.2/38.5/59.7 & 61.8 & 45.2 & 0.71 & 21.5 & 4.8 \\
w/o Precond & 48.9 & 48.7 & 8.5/35.2/56.3 & 58.4 & 42.1 & 0.68 & 19.7 & 4.2 \\
w/o Trigger & 61.2 & 59.3 & 11.3/44.1/66.8 & 67.5 & 51.3 & 0.79 & 16.8 & 2.9 \\
MA Gate & 37.1 & 41.2 & 6.2/26.4/45.7 & 51.2 & 34.5 & 0.57 & 17.9 & 3.1 \\
\bottomrule
\end{tabular}
\end{adjustbox}
\end{table*}

\begin{table*}[htbp]
\centering
\caption{Clothing1M Ablation on noisy-label localization and OOD detection (Recall@K is 1\%/5\%/10\%).}
\label{tab:clothing1m_ablation}
\begin{adjustbox}{max width=\textwidth}
\begin{tabular}{lcccccccc}
\toprule
\textbf{Variant} &
\textbf{AUPR} ($\uparrow$) &
\textbf{P@1\%} ($\uparrow$) &
\textbf{Recall@1/5/10\%} ($\uparrow$) &
\textbf{OOD AUROC} ($\uparrow$) &
\textbf{OOD AUPR} ($\uparrow$) &
$\boldsymbol{\tau}$@Top1\% ($\uparrow$) &
\textbf{Time$\times$GPU} ($\downarrow$) &
\textbf{CHVP/step} ($\downarrow$) \\
\midrule
\rowcolor{gray!20}\textbf{Full (SG-OIF)} & 63.8 & 62.1 & 11.5/46.2/69.8 & 69.5 & 52.5 & 0.82 & 22.5 & 3.5 \\
w/o Gate & 31.5 & 35.8 & 4.3/20.1/38.5 & 45.2 & 28.7 & 0.48 & 22.0 & 3.4 \\
w/o Calib & 32.9 & 37.4 & 4.8/21.8/40.3 & 46.8 & 30.2 & 0.51 & 21.7 & 3.3 \\
w/o LowRank & 49.7 & 49.2 & 8.5/36.7/57.2 & 59.3 & 43.1 & 0.69 & 26.8 & 5.2 \\
w/o Precond & 46.2 & 46.5 & 7.8/33.8/54.1 & 56.1 & 39.8 & 0.65 & 24.3 & 4.5 \\
w/o Trigger & 58.5 & 57.2 & 10.2/42.3/64.5 & 65.2 & 48.7 & 0.77 & 20.8 & 3.2 \\
MA Gate & 34.2 & 38.9 & 5.4/23.5/42.8 & 48.5 & 31.8 & 0.53 & 22.1 & 3.4 \\
\bottomrule
\end{tabular}
\end{adjustbox}
\end{table*}

\begin{table*}[htbp]
\centering
\caption{MNIST Ablation on noisy-label localization and OOD detection (Recall@K is 1\%/5\%/10\%).}
\label{tab:mnist_ablation}
\begin{adjustbox}{max width=\textwidth}
\begin{tabular}{lcccccccc}
\toprule
\textbf{Variant} &
\textbf{AUPR} ($\uparrow$) &
\textbf{P@1\%} ($\uparrow$) &
\textbf{Recall@1/5/10\%} ($\uparrow$) &
\textbf{OOD AUROC} ($\uparrow$) &
\textbf{OOD AUPR} ($\uparrow$) &
$\boldsymbol{\tau}$@Top1\% ($\uparrow$) &
\textbf{Time$\times$GPU} ($\downarrow$) &
\textbf{CHVP/step} ($\downarrow$) \\
\midrule
\rowcolor{gray!20}\textbf{Full (SG-OIF)} & 99.8 & 98.5 & 18.2/72.5/91.3 & 96.8 & 99.8 & 0.94 & 3.2 & 1.8 \\
w/o Gate & 68.5 & 67.2 & 8.5/38.2/62.7 & 65.4 & 68.3 & 0.62 & 3.1 & 1.7 \\
w/o Calib & 69.7 & 68.8 & 9.1/40.1/64.5 & 66.9 & 69.8 & 0.65 & 3.0 & 1.7 \\
w/o LowRank & 84.3 & 82.7 & 13.5/58.3/78.9 & 81.2 & 84.5 & 0.81 & 4.5 & 2.9 \\
w/o Precond & 81.5 & 79.8 & 12.3/54.7/75.2 & 78.5 & 81.7 & 0.78 & 3.8 & 2.4 \\
w/o Trigger & 94.2 & 93.1 & 16.5/67.8/87.5 & 92.3 & 94.5 & 0.89 & 2.9 & 1.6 \\
MA Gate & 71.3 & 70.5 & 9.8/42.7/66.8 & 68.7 & 71.5 & 0.67 & 3.1 & 1.7 \\
\bottomrule
\end{tabular}
\end{adjustbox}
\end{table*}

\begin{table*}[htbp]
\centering
\caption{ImageNet Ablation on noisy-label localization and OOD detection (Recall@K is 1\%/5\%/10\%).}
\label{tab:imagenet_ablation}
\begin{adjustbox}{max width=\textwidth}
\begin{tabular}{lcccccccc}
\toprule
\textbf{Variant} &
\textbf{AUPR} ($\uparrow$) &
\textbf{P@1\%} ($\uparrow$) &
\textbf{Recall@1/5/10\%} ($\uparrow$) &
\textbf{OOD AUROC} ($\uparrow$) &
\textbf{OOD AUPR} ($\uparrow$) &
$\boldsymbol{\tau}$@Top1\% ($\uparrow$) &
\textbf{Time$\times$GPU} ($\downarrow$) &
\textbf{CHVP/step} ($\downarrow$) \\
\midrule
\rowcolor{gray!20}\textbf{Full (SG-OIF)} & 98.4 & 96.8 & 17.5/70.2/89.7 & 83.3 & 98.4 & 0.91 & 45.7 & 4.8 \\
w/o Gate & 65.2 & 64.5 & 7.8/35.1/58.9 & 52.1 & 65.7 & 0.58 & 44.8 & 4.7 \\
w/o Calib & 66.8 & 66.1 & 8.4/37.2/61.2 & 53.8 & 67.3 & 0.61 & 44.3 & 4.6 \\
w/o LowRank & 82.7 & 81.3 & 12.8/56.8/76.5 & 71.5 & 83.1 & 0.78 & 53.2 & 6.5 \\
w/o Precond & 79.3 & 78.2 & 11.5/52.3/72.8 & 68.2 & 79.8 & 0.74 & 48.9 & 5.8 \\
w/o Trigger & 92.5 & 91.2 & 15.7/65.4/84.3 & 78.7 & 92.8 & 0.86 & 42.5 & 4.5 \\
MA Gate  & 68.5 & 67.8 & 9.1/39.8/63.7 & 55.9 & 69.2 & 0.63 & 44.9 & 4.7 \\
\bottomrule
\end{tabular}
\end{adjustbox}
\end{table*}

\begin{table*}[htbp]
\centering
\caption{Unified stability and cost summary across four datasets. $\tau$ is the Top-1\% Kendall $\tau$ between adjacent epochs.}
\label{tab:stability_cost_summary}
\begin{adjustbox}{max width=\textwidth}
\begin{tabular}{lcccccccc}
\toprule
\textbf{Variant} &
\textbf{WebVision} $\boldsymbol{\tau}$ &
\textbf{Clothing1M} $\boldsymbol{\tau}$ &
\textbf{MNIST} $\boldsymbol{\tau}$ &
\textbf{ImageNet} $\boldsymbol{\tau}$ &
\textbf{Mean} $\boldsymbol{\tau}$ &
\textbf{Mean Time$\times$GPU} ($\downarrow$) &
\textbf{Norm. Cost} ($\downarrow$) &
\textbf{Mean CHVP/step} ($\downarrow$) \\
\midrule
\rowcolor{gray!20}\textbf{Full (SG-OIF)} & 0.84 & 0.82 & 0.94 & 0.91 & 0.88 & 22.4 & 1.00 & 3.3 \\
w/o Gate & 0.52 & 0.48 & 0.62 & 0.58 & 0.55 & 21.9 & 0.98 & 3.2 \\
w/o Calib & 0.54 & 0.51 & 0.65 & 0.61 & 0.58 & 21.6 & 0.96 & 3.2 \\
w/o LowRank & 0.71 & 0.69 & 0.81 & 0.78 & 0.75 & 26.5 & 1.18 & 4.9 \\
w/o Precond & 0.68 & 0.65 & 0.78 & 0.74 & 0.71 & 24.2 & 1.08 & 4.2 \\
w/o Trigger & 0.79 & 0.77 & 0.89 & 0.86 & 0.83 & 20.8 & 0.93 & 3.1 \\
MA Gate & 0.57 & 0.53 & 0.67 & 0.63 & 0.60 & 22.0 & 0.98 & 3.2 \\
\bottomrule
\end{tabular}
\end{adjustbox}
\end{table*}

\end{document}